\documentclass[onefignum,draft,onetabnum,reqno,10pt]{siamart171218}

\usepackage[T1]{fontenc}
\usepackage{amssymb}

\usepackage{lipsum}
\usepackage{amsfonts}
\usepackage{graphicx}
\usepackage{epstopdf}
\usepackage{algorithmic}
\ifpdf
  \DeclareGraphicsExtensions{.eps,.pdf,.png,.jpg}
\else
  \DeclareGraphicsExtensions{.eps}
\fi

\usepackage{dsfont}
\newcommand{\N}{\mathbb{N}}
\newcommand{\R}{\mathbb{R}}

\newsiamremark{remark}{Remark}
\newsiamremark{hypothesis}{Hypothesis}
\crefname{hypothesis}{Hypothesis}{Hypotheses}
\newsiamthm{claim}{Claim}

\headers{Analysis of NeuralODE}{E.\ C.\ Ehrhardt, T.\ J.\ Riedlinger and H.\ Gottschalk }

\title{Numerical and statistical analysis of NeuralODE with Runge-Kutta time integration
}
\author{
    Emily C. Ehrhardt\thanks{Department of Mathematics, Technical University of Berlin, Germany (\email{emily.ehrhardt24@imperial.ac.uk}, \email{gottschalk@math.tu-berlin.de}, \email{riedlinger@math.tu-berlin.de}}
    \and
    Hanno Gottschalk\footnotemark[1]
    \and
    Tobias J. Riedlinger\footnotemark[1]}

\usepackage{amsopn}

\makeatletter
\newcommand*{\addFileDependency}[1]{
  \typeout{(#1)}
  \@addtofilelist{#1}
  \IfFileExists{#1}{}{\typeout{No file #1.}}
}
\makeatother

\newcommand*{\myexternaldocument}[1]{%
    \externaldocument{#1}%
    \addFileDependency{#1.tex}%
    \addFileDependency{#1.aux}%
}

\usepackage{tikz}
\usetikzlibrary{shapes.geometric}
\usetikzlibrary{arrows.meta} 
\usetikzlibrary{calc,matrix,positioning}

\newtheorem{assumptions}[theorem]{Assumptions}

\newcommand{\ThetaLW}{\Theta_{L, W}}
\newcommand{\Lip}{\mathrm{Lip}}
\newcommand{\e}{\mathrm{e}}
\newcommand{\LamRK}{\Lambda_{\mathrm{RK}}}
\newcommand{\id}{\mathrm{id}}

\ifpdf
\hypersetup{
  pdftitle={Numerical and statistical analysis of NeuralODE with Runge-Kutta time integration },
  pdfauthor={E. Ehrhardt, H. Gottschalk, T. Riedlinger}
}
\fi

\myexternaldocument{ex_supplement}

\begin{document}

\maketitle

\begin{abstract}
NeuralODE is one example for generative machine learning based on the push forward of a simple source measure with a bijective mapping, which in the case of NeuralODE is given by the flow of a ordinary differential equation. Using Liouville's formula, the log-density of the push forward measure is easy to compute and thus NeuralODE can be trained based on the maximum Likelihood method such that the Kulback-Leibler divergence between the push forward through the flow map and the target measure generating the data becomes small.

In this work, we give a detailed account on the consistency of Maximum Likelihood based empirical risk minimization for a generic class of target measures. In contrast to prior work, we do not only consider the statistical learning theory, but also give a detailed numerical analysis of the NeuralODE algorithm based on the 2nd order Runge-Kutta (RK) time integration. Using the universal approximation theory for deep ReQU networks, the stability and convergence rated for the RK scheme as well as metric entropy and concentration inequalities, we are able to prove that NeuralODE is a probably approximately correct (PAC) learning algorithm. 
\end{abstract}

\begin{keywords}
  NeuralODE $\bullet$ Runge Kutta time integration $\bullet$ numerical analysis $\bullet$ statistical learning theory.
\end{keywords}

\begin{AMS}
 68T05, 34A34, 65L06.
\end{AMS}

\section{Introduction}
NeuralODE  stands for a generative machine learning algorithm that models a transport map between a source and a target probability measure as the flow of an ordinary differential equation (ODE)~\cite{neuralordinarydifferentialequations}. The vector field on the right hand side of the ODE is represented by a neural network with trainable weights, which  makes the ODE neural. From the theory of optimal transport, it is known that under rather general conditions vector fields and flows exist that transport the source measure exactly on the target measure~\cite{SantambrogioFilippo2015OTfA}. Recently, under suitable regularity conditions on the source and target densities, the existence of $C^k$ vector fields has been proven~\cite{marzouk2024distribution}. 

The remaining question is, given that such vector fields exist, can they be learned by neural networks and can new samples be generated by solving ODEs numerically. As the target measure is generally unknown, we have to use data drawn from the target measure to control the error between target and learned target measure. NeuralODE permit an explicit representation of the density of the learned target measure by Liouville's  formula, therefore the training of NeuralODE is likelihood based. However, Liouville's formula contains an integral over the divergence of the vector field along trajectories. A rigorous analytical treatment in the spirit of statistical learning theory thus requires control over the first derivatives of the neural vector fields. Therefore, recent results on the universal approximation in $C^1$ using deep $\mathtt{ReQU}$ networks with activation $\mathtt{ReQU}(x)=\max{\{0,x\}}^2$~\cite{simultaneous_approx} are employed to prove sufficiently strong universal approximation properties for densities generated by neural flows.

In practice, NeuralODE have to be solved numerically. It is therefore necessary to investigate the interplay between the numerical analysis of the ODE and the statistical learning of the vector field using the neural network. In this work, we analyze these two aspects and thereby for the first time give a proof for successful leaning that also takes the numerical analysis into account. More specifically, we conduct this program for second order Runge-Kutta time integration~\cite{SolvingODE}. Here one has to guarantee that a single time step does not spoil the resulting map being a differentiable bijection, because otherwise the change of variables formula for densities breaks down. Here we prove that for sufficiently small step size this can be avoided and numerical solutions to flow maps still represent diffeomorphisms.   

The mathematical theory for generative models at the time of writing is a relatively young field that combines universal approximation~\cite{cybenko1989approximation,yarotsky2017error,simultaneous_approx} and statistical learning theory~\cite{shalev2014understanding} and non parametric and infinite dimensional statistics~\cite{gine2016mathematical}.

For generative learning, often a weaker notion of the universal approximation property is employed. A class of functions $\mathcal{F}$ from $[0,1]^d$ to $[0,1]^d$ is defined to be a distributional universal approximator for a given source measure $\nu$ and a class of target measures $\mathcal{T}$, if the closure of $\mathcal{F}_*\nu=\{\Phi_*\nu:\Phi\in\mathcal{F}\}$ with respect to some topology (e.g.\ weak convergence in law) contains $\mathcal{T}$.  Distributional universal approximation has been proven for a number of generative models, from restricted Boltzmann machines and deep belief networks~\cite{montufar2011refinements}, Generative adversarial neural networks (GAN)~\cite{puchkin2024rates}, invertible neural networks~\cite{ishikawa2023universal,rochau2024minimal}, NeuralODE~\cite{li2022deep} and diffusion models~\cite{oko2023diffusion}.  

The second aspect is statistical learning theory with hypothesis spaces of adaptive size making use of distributional universal approximation to control model misspecification. Here, variance bias tradeoffs have to be balanced to avoid overfitting on the one hand and to keep the induced error from the model misspecification under control. Usually in statistics, metric entropy estimates and concentration inequalities are employed to control the statistical part of the learning error in the large sample limit. Based on convergence rates obtained here, the adaptive growth of the hypothesis space can then be determined. This program has, e.g., successfully been conducted for GAN~\cite{asatryan2023convenient,biau2020some,biau2021some,puchkin2024rates} and diffusion models~\cite{oko2023diffusion}.

The closest paper to this work is~\cite{marzouk2024distribution}, where a detailed account on the statistical learning theory for NeuralODE is developed. While this  work provides similar insights as the present article, there nevertheless exist several differences in the techniques applied and the derived results. First and most importantly,~\cite{oko2023diffusion} does not consider the numerical analysis of the NeuralODE, as we do here. The numerical analyis of NeuralODE has been considered in~\cite{zhu2022numerical}, however only with regard of the flow map, not of Liouville's formula, as needed for statistical learning. Secondly,~\cite{marzouk2024distribution} proves convergence in the Hellinger distance, while here we work with the Kulback-Leibler divergence, which induces a stronger topology than the Hellinger distance. Finally, we develop our regularity analysis of flows along the lines of the Beckmann problem \cite{SantambrogioFilippo2015OTfA} and prove existence and Hölder regularity for the vector field generating the flow, while \cite{marzouk2024distribution} relies on the representation of the Rosenblatt-Knothe map as a flow endpoint \cite{wang2022minimax}. 
In addition, treating numerical integration of the vector field ultimately leads to an exponential convergence rate of the PAC probability in $n$ as opposed to the effectively algebraic rate achieved in \cite{marzouk2024distribution}. However, this comes at the cost of our sample requirement also growing exponentially with $1/\varepsilon$ where $\varepsilon$ is the PAC precision.
Other aspects are similar, e.g. the use of metric entropy and concentration estimates and the use of $\mathtt{ReQU}$ networks for the approximation of vector fields and their first derivatives.      

Our paper is organized in the following way. 
Section 2 provides the existence of flow maps as regular solutions of Beckmann's problem using potential and Schauder theory.
Section 3 introduces neural ODE models and the numerical Runge-Kutta integration scheme of second order.
We introduce \texttt{ReQU} neural networks and a universal approximation result which is vital in our treatment of the model error.
We further describe the learning objective and the setting of statistical learning of distributions and empirical risk minimization together with the error decomposition for generative models.
Section 4 contains bounds on the model error of neural ODE models and we use metric entropy estimates (Dudley's inequality) and concentration (Mc Diarmid's inequality) to control the generalization error. 
Finally, we combine both kinds of error and design an adaptive enlargement for the hypothesis space in terms on network width and depth to prove PAC learning for NeuralODE. 

\section{Existence of generative flows}

\subsection{Flow Maps and Liouville's Formula}\label{sec:liouville-formula}
Many generative algorithms learn transport maps $\Phi:\Omega\to\Omega$ which map a source distribution $\nu$ to a target distribution $\mu$, i.e.  $\mu$, $\nu$ are probability measures on a domain $\Omega\subseteq \R^d$ with Lipschitz boundary and we require $\Phi_*\nu=\mu$ for $\Phi$ to be a transport map, where $\Phi_*\nu(A)=\nu(\Phi^{-1}(A))$ is the image measure of $\nu$ under $\Phi$.  Optimal transport theory deals with the existence and properties of transport maps. NeuralODE, use maps that are endpoints of flows $\Phi(\cdot,t)$ with respect to some vector field $\xi:\Omega\times [0,1]\to \R^d$, i.e.
\begin{equation}
    \label{eq:flow}
    \frac{\mathrm{d}}{\mathrm{d} t}\Phi_{s,t}(x)=\xi(\Phi_t(x),t),~~\Phi_{s,s}(x)=x.
\end{equation}
In the following, we assume that $\xi(\cdot,t)\cdot\vec{\eta}=0$ holds on the boundary $\partial \Omega$ of $\Omega$, where $\vec\eta$ is the outward normal vector field.
The flow endpoint is then defined as $\Phi=\Phi_{0,1}(\cdot)$. 
For a general bijective and differentiable map $\Phi$, the density of $\Phi_*\nu$, given that $\nu=f_\nu\, \mathrm{d}x$ has a probability density $f_\nu$ with respect to the Lebesgue measure $\mathrm{d}x$, the density of $\Phi_*\nu $ is given by the transformation of densities formula \cite[Thm.\ IV.8.9]{wernerEinfuehrungHoehereAnalysis2009a}
\begin{equation}
    \label{eq:transformation_of_densities}
    f_\Phi=f_\nu\circ \Phi^{-1} \, \cdot \left|\det D\Phi^{-1}\right|,
\end{equation}
where $D\Phi^{-1}$ stands for the Jacobi matrix of the function $\Phi^{-1}$ and $\det A$ denotes the determinant of a quadratic matrix $A$. If, in particular, $\Phi$ is a flow endpoint w.r.t. a vector field $\xi(x,t)$ which is differentiable in $x$, Liouville's formula gives a representation for $\det D\Phi_{s,t}$ 
\begin{lemma}{(Representation of Log-Determinant)}\label{lem:logdet-formula}
For $s,t \in \R, k \in \N, k\geq 2$ let $\Phi_{s,t}$ be the flow generated by $\xi \in C^{1}(\R^d \times \R, \R^d)$ with $\xi(\cdot , \tau)$ globally Lipschitz for all $\tau \in \R$. Then,
    \begin{equation}
    \log(|\det(D_y\Phi_{s,t}(y;\xi))|) = \int_s^t \mathrm{div}_y\left(\xi(\Phi_{s,\tau}(y;\xi)), \tau\right) \; \mathrm{d}\tau.
    \end{equation}
\begin{proof}
    It is a well-known fact (e.g., \cite[Sec.\ 3.6.1]{KnaufAndreas2018MPCM}) that the Jacobian $J(y, t) := D_y \Phi_{s,t}(y; \xi)$ satisfies the linear initial value problem
    \begin{equation}
        \frac{\mathrm{d}}{\mathrm{d}t} J(y,t) = A(y,t) \cdot J(y,t), \qquad J(y,s) = \mathrm{id}_{\R^d}    ,
    \end{equation}
    where $A(y,t) := D_y \xi(y,t)$.
    An application of Liouville's formula \cite[Thm.\ 1.2]{HartmanODE} yields $\det(J(y,t) = \det(J(y,s)) \cdot \exp(\int_s^t \mathrm{tr}(A(y,\tau)) \,\mathrm{d} \tau$.
\end{proof}
\end{lemma}
If, in particular, if $\Phi$ is a flow endpoint the integral in Lemma \ref{lem:logdet-formula} ranges from 0 to one. As $\Phi^{-1}$ is the flow endpoint with respect to the vector field $-\xi$, Lemma \ref{lem:logdet-formula} can directly be inserted into the transformation of densities formula \eqref{eq:transformation_of_densities}.

\subsection{Beckmann's problem}
From the previous subsection we have seen how given vector fields transform distributions. Furthermore, by Liouville's formula we recognize the role of differentiability of the flow. This leads to the question, when a flow exists that transforms a given source distribution represented by the density $f_\nu\,\mathrm{d}x$ to a target distribution with density $f_\mu \,\mathrm{d}x$ and what can be said about the regularity the underlying vector field $\xi(\cdot,\tau)$.   
Here we give an account to this problem based on potential theory and Schauder estimates for elliptic partial differential equations (PDE)~\cite{agmon1959estimates,gilbarg1977elliptic}. Our solution is related to the so called Beckmann problem in optimal transport \cite[Chapter 4]{SantambrogioFilippo2015OTfA}.  

In Beckmann's approach, we are looking for a vector field $w:\Omega \to \R^d$ which models the total influx or outflux of a differential volume inside a bounded domain $\Omega\subseteq \R^d$ with Lipschitz boundary $\partial\Omega$.  We consider a time dependent probability density $f_t(x)$ that evolves according to
\begin{equation}
\label{eq:densityFlow}
\frac{\partial}{\partial t} f_t+\textrm{div}_x w(t,\cdot)=0,~~f_0=f_\nu,~~\text{on}~~\Omega~~(a.e.). 
\end{equation}
Let us assume for a moment that some \emph{time independent} vector field exists that fulfills 
\begin{equation}
\label{eq:divergenceConstraint}
\textrm{div}_x w=f_\nu-f_\mu~~\text{on}~~\Omega, ~~\text{and}~~w\cdot \vec\eta=0,~~\text{on}~~\partial\Omega~~(a.e.),    
\end{equation}
where $\vec\eta$ is the outward pointing normal vector field on $\partial\Omega$.
Combining~\eqref{eq:densityFlow} and~\eqref{eq:divergenceConstraint} easily implies
\begin{equation}
\label{eq:interpolateDensities}
f_t=(1-t)f_\nu+tf_\mu,~~t\in[0,1], ~~\text{on}~~ \Omega.    
\end{equation}
If we compare this to the continuity equation for the flow of probability densities caused by the flow of a vector field $\xi(x,t)$ applied to an initial condition $f_\nu(x)$
\begin{equation}
\label{eq:continuityEquation}
\frac{\partial}{\partial t} f_t=\textrm{div}_x (f_t\xi(\cdot,t)),~~f_0=f_\nu,~~\text{on}~~\Omega,
\end{equation}
we see that given $w(x)$, the vector field 
\begin{equation}
\label{eq:beckmannVectorField}
\xi(\cdot,t)=\frac{w}{f_t}
\end{equation}
is well-defined for $t\in[0,1]$, provided $f_\mu(x),f_\nu(x)>0$ on $\Omega$, and fulfills $\xi(x,t)\cdot \vec\eta(x)=0,~\text{for a.e.}~x\in\partial\Omega$ and $t\in[0,1]$, which implies that the probability mass is conserved inside $\Omega$.

The Beckmann problem is to solve
\begin{equation}
\label{eq:min_sq_norm_w}
J(w)=\frac{1}{2}\int_\Omega |w|^2\, \mathrm{d}x \rightarrow \min ~, 
\end{equation}
where the minimization takes place under vector fields $w$ which fulfill the constraint\linebreak \eqref{eq:divergenceConstraint} and $w$ is, e.g., in some Sobolev space $W^{1,p}(\Omega,\R^d)$, $p\in[2,\infty)$, that allows taking a derivative and a trace on $\partial\Omega$ \cite{adams2003sobolev}. While the existence of optimal solutions $w$ can be proven under suitable conditions, the known regularity properties of $w$ are not yet strong enough to assure the existence of flows connected to $\xi(x,t)$ \cite[Chapter 4]{SantambrogioFilippo2015OTfA}, see however~\cite{dweik2022w1,Lorenz2022} for some recent progress in more irregular setting that we consider here. 
In the following we identify a sufficiently strong  set of assumptions that allows us to derive Hölder regularity for the vector field $w$. 

\subsection{Unconstrained formulation of the Beckmann problem}

Here we pass from the constraint formulation of Beckmann's problem to an unrestricted formulation using Langrangian multipliers $u:\Omega\to\R$ and $v:\partial\Omega\to \R$. We thus consider the Lagrangian functional
\begin{equation}
    \label{eq:Lagrangian}
    \mathcal{L}(w,u,v)=\frac{1}{2}\int_\Omega |w|^2\, \mathrm{d}x+\int_\Omega (\mathrm{div}_x w-f_\nu+f_\mu)\,u\, \mathrm{d}x+\int_{\partial\Omega} (\vec\eta \cdot w)\, v\, \mathrm{d}S,
\end{equation}
where $dS$ stands for the induced surface volume element on $\partial\Omega$. If $u$ and $v$ are chosen from sufficiently large linear function spaces, the constraint minimization of \eqref{eq:min_sq_norm_w} is equivalent to 
\begin{equation}
    \label{eq:unconstraint_min_w}
    \min_w \sup_{u,v} \mathcal{L}(w,u,v)
\end{equation}
In fact, should one of the constraints in \eqref{eq:divergenceConstraint} be violated, one could find suitable Lagrange multipliers $u$, $v$ such that the last two terms in \eqref{eq:Lagrangian}. Thus the supremum of these two terms in $u$ and $v$ is infinite, which implies that such a $w$ can not be the minimizer.

We now give a rigorous mathematical formulation of \eqref{eq:unconstraint_min_w} based on the Hölder regularity of the densities of source and target measures.
To this purpose, let \linebreak $C^{k,\alpha}(\bar\Omega,\R^q)$ be the set of $k$ times differentiable functions from $\bar \Omega$ to $\R^q$ with $\alpha$-H\"older continuous $k$-th derivative, $k\in\N_0$, $\alpha \in (0,1]$, $q\in\N$. We utilize the following convention on the Hölder norms
\begin{align}
\begin{split}
\label{eq:HoelederNorms}
 \| f\|_{C^{k,\alpha}}&=\max \left\{\|D^\beta f_j\|_{C(\bar\Omega,\R)}:j=1,\ldots,q,\, |\beta|\leq k\right\} \\
 &+\max \left\{\sup_{x \neq x'\in\bar \Omega}\frac{\left|D^\beta (f_j(x)-f_j(x'))\right|}{|x-x'|^\alpha}:j=1,\ldots,q,\, |\beta|=k\right\},
\end{split}
\end{align}
where $\beta=(\beta_1,\ldots\beta_d)\in\N_0^d$ is a multi index  with degree with $|\beta|=\sum_{j=1}^d\beta_j$ and $D^\beta =\frac{\partial^{|\beta|}}{\partial x_1^{\beta_1}\cdots \partial x_d^{\beta_d}}$ the corresponding differential operator.

If $\Omega$ has a $C^{k,\alpha}$ boundary, i.e. the boundary can be straightened by $C^{k,\alpha}$ hemisphere transformations, then we write $g\in C^{k,\alpha}(\partial\Omega,\R)$ if $g:\partial\Omega\to \R$ such that for any of the hemisphere transformations, the mapping from the boundary of the hemisphere to $\partial \Omega$ composed with the hemisphere transformation restricted to the straightened boundary is locally given by a function from $C(\R^{d-1},\R^d)$, see~\cite{agmon1959estimates,gilbarg1977elliptic} for further details.  

\begin{assumptions}
Let $k\in \N_0$ and $\alpha\in(0,1)$.
    \begin{itemize}
        \item [A1)] Let $\Omega \subseteq \R^d$ be bounded with $C^{k+2,\alpha}$ boundaries.
        \item[A2)] Let $f_\mu>0$ and $f_\nu>0$ on $\Omega$.
        \item[A3)] Let $f_\nu, f_\mu \in C^{k,\alpha}(\bar\Omega,\R)$ for some $k\in\N_0$ and $\alpha\in(0,1]$.
    \end{itemize}
\end{assumptions}

Assumption A2) clearly implies that
\begin{equation}
    \label{eq:kappaPositive}
    \kappa:=\min\left\{\inf_{x\in\bar \Omega} f_\mu(x),\inf_{x\in\bar \Omega} f_\nu(x)\right\}>0,
\end{equation}
as the continuous functions $f_\mu$ and $f_\nu$ attain their minimum on the compact set $\bar \Omega$.

Let us now specify suitable spaces for the optimization variables, i.e. we chose $w\in W^{1,p}(\Omega,\R^d)$, $p\geq 2$, $u\in C^{k+2,\alpha}(\bar\Omega,\R)$ and $v\in C^0(\partial\Omega,\R)$. By the Sobolev embedding theorem \cite{adams2003sobolev}, $w$ has a continuous extension to $\bar \Omega$ if $p\geq d$ and admits a trance with values in $L^q(\partial\Omega)$ for $p<d$ if $1\leq q\leq \frac{(d-1)p}{(d-p)}$ if $1\leq p<d$. Furthermore, under the given conditions the integration by parts formula holds for the first derivatives of $w$ multiplied by $u$ \cite{adams2003sobolev}. 

Let us thus investigate the first order optimality conditions of the optimization problem \eqref{eq:unconstraint_min_w}
\begin{subequations}
\begin{align}
    \label{eq:first_order_opt_a}
    0&= \frac{\delta}{\delta v}\mathcal{L}(w,u,v)=\vec\eta\cdot w~~\text{on}~~\partial\Omega ~~(a.e.)\\
    \label{eq:first_order_opt_b}
    0&= \frac{\delta}{\delta u}\mathcal{L}(w,u,v)=\mathrm{div}_x w-f_\nu+f_\mu=0~~\text{on}~~\Omega ~~(a.e.)\\
    \label{eq:first_order_opt_c}
    0&= \frac{\delta}{\delta w}\mathcal{L}(w,u,v)=w -\nabla u=0~~\text{on}~~\Omega ~~(a.e.)
\end{align}
\end{subequations}
To derive \eqref{eq:first_order_opt_c}, we used integration by parts by the divergence theorem and \eqref{eq:first_order_opt_a}, $\nabla$ stands for the gradient operator. Since $p\geq 2$, all terms can be interpreted as Fr\'echet derivatives of $\mathcal{L}$ in the respective spaces.  

Combining  \eqref{eq:first_order_opt_a}-- \eqref{eq:first_order_opt_c}, we see that the Lagrangian function $u:\Omega\to\R$ has to solve the Poisson equation with Neumann boundary conditions
\begin{equation}
\label{eq:Poisson}
\Delta u=f_\nu-f_\mu,~~~x\in\Omega, ~~\nabla u\cdot \vec\eta=0,~~~x\in\partial\Omega,~~w=\nabla u.
\end{equation}
Here $\Delta=\sum_{j=1}^d\frac{\partial^2}{\partial x_j^2}$ denotes the Laplace operator.

\subsection{Potential theory and Schauder estimates}

In this section we derive the existence and regularity of the solution $u$ of the first order optimality conditions of the Beckmann problem in the form \eqref{eq:Poisson}  using elliptic regularity theory~\cite{agmon1959estimates,gilbarg1977elliptic}.

Let us first recall a result on the existence of strong solutions of the Poisson equation with Neumann boundary conditions via the Fredholm alternative:

\begin{theorem}[Existence and $C^{2,\alpha}$ Schauder estimate~\cite{nardi2015schauder}]
\label{theorem:Schauder}
Let $d>2$, $\alpha\in(0,1)$ and $\Omega$ fulfill assumption A1) for $k=0$. Let $f\in C^{0,\alpha}(\Omega,\R)$ be given and $g:\partial \Omega\to \R$ be in $C^{1,\alpha}(\partial\Omega,\R)$. If
\begin{equation}
    \label{eq:FredholmAlternative}
    \int_\Omega f\, \mathrm{d}x+\int_{\partial\Omega} g \,\mathrm{d} S=0,
\end{equation}
then there exits a solution $u\in C^{2,\alpha}(\overline{\Omega},\R)$ to the Poisson equation $\Delta u=f$ on $\Omega$ with Neumann boundary conditions $\nabla u  \cdot \vec\eta = g$ on $\partial \Omega$ which is unique up to a constant.

Furthermore, there exists a constant $C=C(\Omega,\alpha)$ such that
\begin{equation}
\label{eq:2ndOrderSchauder}
\left\| u-\frac{1}{|\Omega|}\int_\Omega u\,\mathrm{d}x\right\|_{C^{2,\alpha}(\Omega,\R)}\leq C(\Omega,d,\alpha)\left(\| f\|_{C^{0,\alpha}(\Omega,\R)}+\|g\|_{C^{1,\alpha}(\partial\Omega,\R)}\right).
\end{equation}
\end{theorem}

\cref{theorem:Schauder} is not yet fully satisfactory as it does not exploit higher regularity of $f=f_\nu-f_\mu$ in case that $f_\mu,f_\nu\in C^{k,\alpha}(\Bar\Omega,\R) $ for $k\geq 1$.
We therefore use classical Schauder estimates for elliptic PDE to derive higher regularity. Application of the classical Schauder estimate by Agmon, Douglis and Nirenberg  to the situation at hand gives the following.

\begin{theorem}[$C^{k,\alpha}$ Schauder estimate~\cite{agmon1959estimates}]
    \label{theorem:ShauderHigherOrder}
 Let $d>2$ and $k\in \N$ and let $f\in C^{k,\alpha}(\bar\Omega,\R)$. Let furthermore $\Omega$ be bounded and fulfill assumptions   A1). If there exists a classical solution to the Poisson equation $\Delta u=f$ on $\Omega$ with Neumann boundary conditions $\nabla u\cdot \vec\eta=g$ on $\partial\Omega$, then actually $u\in C^{2+k,\alpha}(\bar \Omega,\R)$ and the following Schauder estimate holds
 \begin{equation}
     \label{eq:higherSchauder}
    \| u\|_{C^{2+k,\alpha}(\bar\Omega,\R)}\leq C_1(\Omega,d,k,\alpha)\left(\|f\|_{C^{k,\alpha}(\bar\Omega,\R)}+\|g\|_{C^{k+1,\alpha}(\bar\Omega,\R)}\right). 
 \end{equation}
\end{theorem}
\begin{proof}
During this proof, we use the notation of~\cite{agmon1959estimates}. We apply Theorem 7.3 of this paper in the setting $m=1$, $l_0=2$ and $l=k+2$. All regions $\mathfrak{A}$ and $\mathcal{D}$ are set to $\Omega$ which fulfills the boundary regularity required by the theorem by Assumption A1). We next consider the (principal) symbol $L(\overline{p},p_{d})=|\overline{p}|^2+p_d^2$ in the formal variables $p=(\bar p,p_d)\in\R^{d-1}\times \R$ where we use local coordinates defined by the Hemispehere transformations such that the boundary is located in direction $p_d$. As required in \cite[p. 632]{agmon1959estimates}, for $\overline{p}\not=0$, the equation $L(\overline{p},p_{d})=0$ has exactly one root with positive imaginary part in $p_d$, namely $i|\bar p|$. Also condition (1.1) on the same page is trivially fulfilled. The complementing condition on page 633 of this work is on linear independence of the coefficients of diverse boundary condition operators is trivially fulfilled, as we only have one single boundary operator. The smoothness and boundedness of the boundary condition operator is just given be the regularity of $\vec\eta$. Thus the conditions of Theorem 7.3 in our case are fulfilled and we obtain the Schauder estimate for any solution $\tilde u$
\begin{equation}
    \| \tilde u\|_{C^{2+k,\alpha}(\bar\Omega,\R)}\leq C'(\Omega,d,k,\alpha)\left(\|f\|_{C^{k,\alpha}(\bar\Omega,\R)}+\|g\|_{C^{k+1,\alpha}(\bar\Omega,\R)}+\|\tilde u\|_{C^0(\bar\Omega,\R)}\right). 
\end{equation}
 Let us now consider the family of solutions $\tilde u\rightarrow \tilde u+c=u$ and choose $c$ such that $\int_\Omega u\, \mathrm{d}x=0$. Now use Theorem \ref{theorem:Schauder} to upper bound the $\|u\|_{C^0(\bar\Omega)}$- term on the right hand side by the right hand side of \eqref{eq:2ndOrderSchauder}. Setting $C_1(\Omega,d,k,\alpha)=C(\Omega,d,\alpha)+C'(\Omega,d,k,\alpha)$ and using the obvious inequalities between $C^{k,\alpha}$ H\"older norms for different $k$, we derive the assertion.   
\end{proof}

\subsection{Regular solutions to Beckmann's problem}
The above two theorems can now be applied to our situation with $g=0$ and $f=f_\nu-f_\mu$ in order to obtain a Hölder regular solution to the Beckmann problem. We also mention the derived regularity for the vector field $\xi$ and the flow $\Phi_t$. The following is the main theorem of this section.
\begin{theorem}
\label{theorem:Ck_alphaForV}
    Let the assumptions A1)--A3) be fulfilled for $k\in \N_0$. Then
    \begin{itemize}
        \item [(i)] There exits a potential field $u\in C^{k+2,\alpha}(\overline{\Omega},\R)$  such that $w=\nabla u$ fulfills~\eqref{eq:divergenceConstraint} and $w\in C^{k+1}(\bar \Omega,\R)$ with
        \begin{equation}
            \|w\|_{C^{k+1,\alpha}(\bar\Omega,\R^d)}\leq C_1(\Omega,d,k,\alpha)\,\|f_\mu-f_\nu\|_{C^{k,\alpha}(\bar\Omega,\R)}     
        \end{equation}
        \item[(ii)] In $W^{1,p}(\Omega,\R^d)$, $p\geq 2$, $w=\nabla u$ is the unique solution to the Beckmann problem  \eqref{eq:min_sq_norm_w} under the constraints. \eqref{eq:divergenceConstraint}.
        \item[(iii)] The vector field $\xi$ defined in~\eqref{eq:beckmannVectorField} lies in $C^{k,\alpha}(\Omega,\R^d)$ with $C^{k,\alpha}$-H\"older norm no larger than
        \begin{equation}
        \hspace{-.15cm} \|\xi\|_{C^{k,\alpha}(\bar\Omega\times[0,1],\R^d)}\leq C_2(\Omega,d,k,\alpha) \left(\frac{\max\{\|f_\nu\|_{C^{k,\alpha}(\bar\Omega,\R)},\|f_\mu\|_{C^{k,\alpha}(\bar\Omega,\R)}\}}{\kappa}\right)^{2^k+5}, 
        \end{equation}
        where $\kappa>0$ is defined in \eqref{eq:kappaPositive}.
        \item[(iv)] If $k\geq 1$, $\xi$ generates a flow $\Phi_{0,t}(\xi)$ with flow endpoint $\Phi(\xi)$ which is a transport map, i.e. $\Phi_*\nu=\mu$. Furthermore, $\Phi_t,\Phi\in C^{k,\alpha}(\bar\Omega,\R^d)$. 
    \end{itemize}  
\end{theorem}

\begin{proof}
 (i) We  apply Theorem \ref{theorem:Schauder} for $g=0$ since by assumption A4) and $f=f_\nu-f_\mu$,  $\int_\Omega f\, \mathrm{d}x=0$ meets the conditions of Theorem \ref{theorem:Schauder}. This gives us $u\in C^{2,\alpha}(\bar\Omega,\R^d)$. For $k\geq 1$, apply Theorem \ref{theorem:ShauderHigherOrder} in addition.  From the $C^{k+2,\alpha}$ regularity of $u$, the $C^{k+1,\alpha}$ regularity of $w$ follows.

 (ii) Clearly, $w\in W^{1,p}(\Omega,\R^d)$ where $w$ is the vector field from (i). Let $w'\in W^{1,p}(\Omega,\R^d)$ fulfill \eqref{eq:divergenceConstraint} and assume $w'\not=w$. Obviously, $j(\tau)= J((1-\tau)w+\tau w')$ is a second order polynomial with non vanishing quadratic coefficient $J(w-w')>0$. Thus, $j(\tau)$ has a global and unique minimum on $\R$ where the first derivative vanishes. By the 1st order optimality condition \eqref{eq:first_order_opt_c} evaluated in the direction $w'-w$, the first derivative vanishes for $\tau=0$. Thus, $J(w)=j(0)<j(1)=J(w')$. 
 (iii)  $f_t$ is linear affine in $t$ and thus $f_t(x)$ is a Hölder function in $C^{k,\alpha}(\overline{\Omega}\times [0,1],\R)$ by assumption A3). Also $f_t(x)\geq \kappa$ by assumption A2). We can thus apply (i) and \cite[Proposition A7]{asatryan2023convenient} which provide Hölder continuity for quotients and products to conclude. 

 (iv) This follows from (ii) and the fact that for $k\geq 1$, the function $\xi \in C^{k,\alpha}(\bar\Omega\times[0,1],\R)$ is differentiable with bounded first derivative. In particular, it is Lipschitz. By the condition $w\cdot \vec\eta=0$ on $\partial\Omega$, we also obtain $\xi\cdot \vec\eta=0$ at the boundary. Hence the flow $\Phi_{0,t}(x)$ applied on some point $x\in\Omega$ never leaves $\Omega$. Along the trajectory, the Lipschitz constant of $\xi$ thus is bounded. We can thus apply standard ODE theory~\cite{HeuserHarro} to prove global existence of a flow $\Phi_{0,t}(x)$ for $x\in\Omega$ and $t\in[0,1]$. That the flow endpoint is a transport map follows from~\eqref{eq:densityFlow}--\eqref{eq:beckmannVectorField}. The $C^{k,\alpha}$-regularity of the flow follows from \cite[Theorem 6]{effland2020convergence} for $k=1$. As the authors remark in their proof, their argument can be easily iterated by applying it to equations like
 \begin{equation}
 \frac{\mathrm{d}}{\mathrm{d}t}\left(D_x \Phi_{s,t}(x)\right)=D_x \xi(\Phi_{s,t}(x),t) \left(D_x \Phi_{s,t}(x)\right)
 \end{equation}
 and higher order analogues, to derive higher $C^{k,\alpha}$-Hölder regularity for $k\geq 2$.
\end{proof}

\section{Generative Learning with Neural ODEs}
Generative learning algorithms aim at sampling from a probability distribution $\mu \in \mathcal{M}_{+}^{1}(\Omega)$ (called “target measure”) over some domain $\Omega \subset \mathbb{R}^{d}$ after having observed a finite amount of data $X_{1}, \ldots, X_{n} \sim \mu$.
Oftentimes, this is done by modeling a parametric distribution $\mu_{\theta} \in \mathcal{M}_{+}^{1}(\Omega)$ where $\theta \in \Theta \subset \mathbb{R}^{q}$ such that $\mu_{\theta}$ is close to $\mu$ in some metric or distance function $\mathrm{d}$ on $\mathcal{M}_{+}^{1}(\Omega)$.
A variety of models (such as GANs, normalizing flows, VAEs or neural ODEs) achieve approximate sampling from $\mu$ by modeling a transport map $\Phi^{\theta}:\Omega \to \Omega$ such that $\Phi_{*}^{\theta} \nu$ is close to $\mu$.
Here, $\nu \in \mathcal{M}_{+}^{1}(\Omega)$ is some probability measure which is easy to generate samples from, called the source measure.

\subsection{Neural ODEs and Runge-Kutta Integration}
The assumption under which neural ODEs are designed is that $\mu= \Phi_{*}\nu$ where $\Phi = \Phi_{0, 1}$ is the end point of a flow map $\Phi_{t_{0},t}$ associated with an ordinary differential equation 
\begin{equation}
    \frac{\mathrm{d}}{\mathrm{d}t} y(t) = \xi(y(t), t)
\end{equation}
given the initial condition $y(t_{0}) = y_{0} \in \mathbb{R}^{d}$ for some vector field $\xi: \mathbb{R}^{d} \times \mathbb{R} \to \mathbb{R}^{d}$ with the solution $y: \mathbb{R} \to \mathbb{R}^{d}, y(t) = \Phi_{t_{0}, t}(y_{0}; \xi)$.
Here, $\Phi_{t_0, t}(\cdot; \xi) = \Phi_{t_0, t}(\xi): y_0 \mapsto y(t)$ denotes the flow map generated by $\xi$.
Since we will regularly consider flow endpoints for time $t = 1$, we shall abbreviate $\Phi(\xi):= \Phi_{0,1}(\xi)$.
What is modeled by the machine learning component is the vector field, i.e., $\eta_{\theta}: \mathbb{R}^{d} \to \mathbb{R}^{d}$ is a parametric model where the flow of the non-autonomous ODE above is oftentimes approximated by a sequence of flows associated with the autonomous ODE generated by $\eta_{\theta}$.
Choosing $\Phi$ to be the solution to an ODE has the advantage that we can easily obtain samples from $\Phi_{*}\nu$ by sampling $z \sim \nu$ and solving backward in time by $x = (\Phi(\eta_\theta))^{-1}(z)$.
However, each evaluation of $\Phi_{0,1}(\eta_\theta)$ requires numerical integration of the vector field $\eta_{\theta}$, e.g., by Runge-Kutta methods introducing numerical errors.

Given a vector field $\xi: \mathbb{R}^d \times \mathbb{R} \to \mathbb{R}^d$ and an initial condition $(t_0,y_0) \in \mathbb{R}\times \mathbb{R}^d$,
the explicit Runge-Kutta method of second order is defined by
\begin{equation}
    \Psi^{2,h}_{t_0}(y_0;\xi) := y_0 + h \cdot \xi \left( y_0 + \frac{h}{2} \cdot \xi (y_0, t_0), t_0 + \frac{h}{2}\right).
\end{equation}
We shall also write $\psi_{t}^h(y;\xi):= h \xi (y + \tfrac{h}{2} \cdot \xi(y,t), t + \tfrac{h}{2})$, i.e., $\Psi_t^{2,h}(y;\xi) = [\mathrm{id} + \psi_t^h(\cdot;\xi)](y)$.
Numerical integration from time $t_0$ to $t > t_0$ can be performed in an interative manner of $m$ steps with step size $h = \frac{t - t_0}{m}$ via composition
\begin{equation}\label{eq:runge-kutta-composition}
    \Psi^{p, h}_{t_0, t}(y_0; \xi) = \Psi^{p,h}_{t_0 + (m - 1)h}(\cdot; \xi) \circ \Psi^{p,h}_{t_0 + (m-2)h}(\cdot; \xi) \circ \cdots \circ \Psi^{p,h}_{t_0} (y_0; \xi),
\end{equation}
where $p = 1,2$.
We call this the $m$-th iterate of the explicit Runge-Kutta method of order $p$.
In the following, we will denote $\Psi_{t_0,t}^h:= \Psi_{t_0, t}^{2,h}$ and the approximate flow endpoint $\Psi^h := \Psi_{0,1}^h$.
The following bound on Runge-Kutta approximations holds.
\begin{theorem}{\cite[Chapter II.3, Thm.\ 3.4]{SolvingODE}}
\label{theorem:discretized_ode-flow-bound}
    Let $\xi \in C^2(\mathbb{R}^d \times \mathbb{R}; \mathbb{R}^d)$ be locally Lipschitz and $t > t_0 \in \mathbb{R}$, $y_0 \in \mathbb{R}^d$ such that $\Phi_{t_0, t}(y_0; \xi)$ exists uniquely.
    Then, for the $m$-th iterate of the explicit Runge-Kutta method of order $p = 1,2$ with step size $h = \frac{t - t_0}{m}$, we have that
    \begin{equation}
        \| \Phi_{t_0,t}(y_0;\xi) - \Psi_{t_0, t}^{p, h}(y_0; \xi)\| \leq h^p \frac{C(\xi)}{\Lip_y(D\xi)} \left( \mathrm{e}^{\Lip_y(D\xi) \cdot |t - t_0|} - 1 \right),
    \end{equation}
    where $\Lip_y(D\xi)$ denotes the Lipschitz constant of the spatial Jacobian of $\xi$ and $C(\xi) = 2\|\xi\|_{C^2}$.
\end{theorem}

\subsection[Neural Networks with ReQU Activation]{Neural Networks with \texttt{ReQU} Activation}
Neural ODEs are typically implemented as neural networks.
Here, we fix our notation in order to apply a central universal approximation result~\cite[Theorem 2]{simultaneous_approx}.
We are mostly concerned with fully connected neural networks which are composed of layers $\lambda_{w,b}^\sigma: \R^l \to \R^m$ of the form 
\begin{equation}\label{eq:layer-function}
    \lambda_{w,b}^\sigma(y) = \sigma(w y + b)
\end{equation}
where $w \in \R^{m \times l}$ and $b \in \R^m$ are called the weight matrix and bias, respectively.
The function $\sigma: \R \to \R$ is understood to be applied to each component and will equal the \texttt{ReQU} function $\mathtt{ReQU}(x) = (\mathtt{ReLU}(x))^2 =  (x\vee 0)^2$ in most cases.
The dimensions $m$ is called the input size and $l$ is called the width of $\lambda_{w,b}$.
For $L \in \N$, we may compose $L$ layers with width parameters $\ell = (\ell_1,\ldots, \ell_L) \in \N^{L}$ to a neural network
\begin{equation}
\eta_\theta : \R^d \to \R^{\ell_L}, \qquad y \mapsto \lambda_{w_L, 0}^{\mathrm{id}} \circ \lambda_{w_{L-1}, b_{L-1}}^{\mathtt{ReQU}} \cdots \circ \lambda_{w_1,b_1}^{\mathtt{ReQU}} (y).
\end{equation}
Here, $w_i \in \R^{\ell_{i-1} \times \ell_i}$ and $b_i \in \R^{\ell_i}$ for $i = 1, \ldots, L$ and input size $\ell_0 := d$ which we fix throughout the remainder of this work.
We call $L$ the depth and $W(\ell):= \max\{\ell_1, \ldots, \ell_L\}$ the width of $\eta_\theta$.
Our learnability proof will only take the depth $L$ and total width $W$ of neural networks into consideration but not the widths $\{\ell_1, \ldots, \ell_L\}$ of the individual layers.
The weight and bias matrices will in the following be regarded as collected into one parameter vector $\theta = (w_1, b_1, \ldots, w_L, b_L) \in \Theta \subset \R^q$ for some fixed $q \in \N$.
We will further denote $\theta_\ell := (w_\ell, b_\ell)$.
We define the class 
\begin{equation}
    \mathcal{N \! N}(L, W) = \{\eta_\theta :\R^d \to \R^W | \ell \in \N^L, W(\ell) \leq W\}
\end{equation}
where for $\ell_j < W$ and $\theta \in \ThetaLW := [-1, 1]^{L \times W \times (W+1)}$, we regard $w_j$ and $b_j$ as embedded in $[-1, 1]^{W \times W}$ and $[-1, 1]^W$, respectively.
In the context of $\eta_\theta \in \mathcal{N\! N}(L,W)$, we shall also write $\Phi_\theta := \Phi(\eta_\theta)$, $\Psi^h_\theta := \Psi^h(\eta_\theta)$ and $\psi_{t}^{\theta,h}(y) := \psi_t^h(y; \eta_\theta)$.

The remainder of this section will establish various Lipschitz properties which will be essential in deriving learning rates for neural ODEs.

\begin{lemma}{(Parameter Lipschitz Continuity of $\theta \mapsto \eta_\theta$ on Bounded Input Sets)}
\label{lem:dnn_lipschitz_in_parameters}
    Let $\Omega \subset \R^d$ be a bounded set and $z_0 \in \Omega$, $L, W \in \N$.
    The mapping $\eta_{(\cdot)}(z_0): \ThetaLW \to \R^W$ with $\theta \mapsto \eta_\theta(z_0)$ is Lipschitz continuous with constant 
    \begin{align}
    \widetilde{\Lambda}_\eta^\Theta = 4 L (2W)^{2^{L + 2} + 2L - 3} \left(\sup_{z \in \Omega} \|z\|_2\right)^{2^{L + 1}}.
    \end{align}
\end{lemma}
\begin{proof}
Let $\theta, \theta' \in \ThetaLW$, $R_0 := \sup_{z \in \Omega} \|z\|_2$ and for any $\ell \in \N$, $\ell \geq 1$, let 
\begin{equation}
    R_\ell := (2W)^2 R_{\ell - 1}^2 = (2W)^{2^{\ell + 1} - 2} R_0^{2^\ell}
\end{equation}
denote the range that can be reached by layer $\ell$.
By the Cauchy-Schwarz inequality, we have for $z_{\ell - 1} \in \overline{B}_r(0)$ with radius $r \geq 1$
\begin{align}
\label{upper_bound_eta_theta}
    \begin{split}
    \|\lambda_{\theta_\ell}^{\mathtt{ReQU}}(z_{\ell-1})\|_2^2 
    \leq& \langle w_\ell z_{\ell-1} + b_{\ell}, w_\ell z_{\ell-1} + b_{\ell} \rangle_2
    \leq \|w_\ell z_{\ell-1} + b_{\ell}\|_2^2\\
    &\leq \left( \|w_\ell \|_2 r + \|b_{\ell}\|_2 \right)^2
    \leq (r + 1)^2 W^2
    \leq (2W)^2 r^2.
    \end{split}
\end{align}
The \texttt{ReQU}-function is Lipschitz continuous on the ball $\overline{B}_{r}(0)$ with Lipschitz constant $2r$, such that with $\widetilde{z}_{\ell - 1} \in \overline{B}_r(0)$ we have the Lipschitz continuity of $z_{\ell - 1} \mapsto \lambda_{\theta_\ell}^{\mathtt{ReQU}}(z_{\ell - 1})$
\begin{align}
\begin{split}\label{eq:layer-lipschitz-z}
    \|\lambda_{\theta_\ell}^{\mathtt{ReQU}}(z_{\ell-1}) - \lambda_{\theta_\ell}^{\mathtt{ReQU}}(\widetilde{z}_{\ell-1}) \|_2 
    &\leq 2 \cdot (2W r) \|w_\ell z_{\ell-1}+ b_{\ell} - (w_\ell \widetilde{z}_{\ell-1} + b_{\ell})\|_2 \\
    &\leq (2W)^2 r\|z_{\ell-1} - \widetilde{z}_{\ell-1}\|_2
\end{split}
\end{align}
and for fixed $z_{\ell - 1}$ also Lipschitz continuity of $\theta_\ell \mapsto \lambda_{\theta_\ell}^{\mathtt{ReQU}}(z_{\ell - 1})$
\begin{align}
    \begin{split}\label{eq:layer-lipschitz-theta}
    \|\lambda_{\theta_\ell}^{\mathtt{ReQU}}(z_{\ell-1}) - \lambda_{\theta_\ell'}^{\mathtt{ReQU}}(z_{\ell-1})\|_2 
    &\leq 2 \cdot (2W R_{\ell-1}) \left( \|w_\ell  - w_\ell' \|_2 \cdot \|z_{\ell-1}\|_2+ \|b_{\ell} - b_{\ell}'\|_2 \right) \\
    &\leq 2 \cdot (2W R_{\ell - 1}) 2 R_{\ell-1} \|\theta-\theta^\prime\|_2 
    = 8 W R_{\ell - 1}^2\|\theta-\theta^\prime\|_2.
    \end{split}
\end{align}
We set $\widetilde{\theta}_\ell := (\theta_1, \dots, \theta_\ell, \theta^\prime_{\ell+1}, \dots, \theta^\prime_{L})$ with $\widetilde{\theta}_0 := \theta'$ and obtain for $z_0 \in \Omega_d$ and $z_\ell := \lambda_{\theta_\ell}^{\mathtt{ReQU}} \circ \dots \circ \lambda_{\theta_1}^{\mathtt{ReQU}}(z_0) \in \overline{B}_{R_\ell}(0)$ by \cref{eq:layer-lipschitz-z,eq:layer-lipschitz-theta}
\begin{align}
    \begin{split}
    \|\eta_\theta(z_0) - \eta_{\theta^\prime}(z_0)\|_2 &\leq \sum_{\ell=1}^L \|\lambda_{\widetilde{\theta}_\ell}^{\mathtt{ReQU}}(z_0) - \lambda_{\widetilde{\theta}_{\ell-1}}^{\mathtt{ReQU}}(z_0)\|_2 \\
    &\leq  \sum_{\ell=1}^L ((2W)^2R_{\ell-1})^{L-\ell} \|\lambda_{\theta_\ell}^{\mathtt{ReQU}} (z_{\ell-1})-\lambda_{\theta^\prime_\ell}^{\mathtt{ReQU}} (z_{\ell-1})\|_2 \\
    &\leq 4 L (2W)^{2L - 1}R_{L-1}^{L+1}\|\theta-\theta^\prime\|_2.
    \end{split}
\end{align}
\end{proof}
In particular, we will approximate vector fields $\xi: \Omega := \Omega_d \times [0,1] \to \R^{d}$ such that $\sup_{z \in \Omega} \|z\|_2 = (1 + \tfrac{1}{4})^{\frac{1}{2}}$.
Also, by \cref{eq:layer-lipschitz-z}, we have the following conclusion:
\begin{corollary}{(Lipschitz Continuity of $z \mapsto \eta_\theta(z)$ on Bounded Input Sets)}\label{cor:dnn-lipschitz-z}
    Let $\Omega\subset \R^d$ be a bounded set, $L,W \in \N$ and $\theta \in \ThetaLW$.
    The mapping $\eta_\theta: \Omega \to \R^W$ with $z \mapsto \eta_\theta(z)$ is Lipschitz continuous with constant $\widetilde{\Lambda}_\eta^\Omega = (2W)^{2L} \prod_{\ell=1}^L R_{\ell - 1}$.
\end{corollary}
\begin{proof}
    Apply \cref{eq:layer-lipschitz-z} to each layer of $\eta_\theta$.
\end{proof}

\begin{remark}
    For future proofs, we note that the Lipschitz constant of \cref{cor:dnn-lipschitz-z} is upper bounded by
\begin{equation}
    \label{eq:dnn-lipschitz-constant-hypothsis-space}
    \widetilde{\Lambda}_\eta^{\Omega} \leq ((2W)^2 R_{L-1})^L =: \Lambda.
\end{equation}
    Further, we also have $\widetilde{\Lambda}_\eta^\Theta \leq \Lambda^2$.
    Bounding the spatial Lipschitz constant of $\eta_\theta$ is particularly relevant for the bijectivity of the solver \eqref{eq:runge-kutta-composition} due to the following result.
\end{remark}
\begin{lemma}{(Bijectivity of Lipschitz Perturbations of the Identity)}
    \label{RK bijective}
    Let $g \colon \R^d \to \R^d$ be a Lipschitz continuous function with $\mathrm{Lip}(g) <1$. Then, $f := \mathrm{id} + g$ is bijective.
\end{lemma}
\begin{proof}
This is a simple generalization of Lemma 6.6.6 in~\cite{Ana2} but we state the proof for sake of completeness.
Let $L := \mathrm{Lip}(g)$.
W.l.o.g. we assume $g(0)=0$. Otherwise we consider $\widetilde{g} = g - g(0)$.
On the one hand, for $x,y \in \R^d$ with $f(x) = f(y)$, we have
\begin{equation}
\|x-y\| = \|(f(x) -  g(x)) - (f(y) - g(y)) \|_2 = \|g(x) - g(y)\| \leq L \|x-y\|,
\end{equation}
i.e., $x = y$.
On the other hand, let $x \in \R^d$ and $r > 0$ such that $\|x\| \leq (1 - L)r$.
Let $F \colon \overline{B_r(0)} \to \overline{B_r(0)}, x \mapsto y - g(x)$. 
By assumption, $F$ is $L$-Lipschitz and for $x \in \overline{B_r(0)}$
\begin{equation}
\|F(x)\| \leq \|y\| + \|g(x) - g(0)\| \leq (1-L)r + L \|x-0\| \leq r.
\end{equation}
Therefore, $F$ is a contraction and there is $z \in \R^d$ such that $z + g(z) = f(z)= y$. 
\end{proof}
\begin{remark}
    The discretized flow of a \texttt{ReQU} vector field $\eta_\theta \in \mathcal{N\!N}(L, W)$ is by \eqref{eq:runge-kutta-composition} bijective on $\Omega_d$ given that each step $\Psi_{\theta, t}^h = \mathrm{id} + \psi_{\theta,t}^h$ is.
    For $x,y \in \Omega_d$ and $t \geq 0$, we have
    \begin{equation}
    \|\psi_{\theta, t}^h(x) - \psi_{\theta,t}^h(y)\|_2 \leq h \Lambda (1 + \tfrac{h}{2} \Lambda) \|x - y\|_2.
    \end{equation}
    Bijectivity of $\Psi_\theta^h$ can then be ensured by the requirement that $h \Lambda < \sqrt{3} - 1$, however, for simplicity we will require
    \begin{equation}\label{eq:assumed-h-bound}
        h < \frac{1}{2 \Lambda}.
    \end{equation}
    This condition couples the capacity of the neural network class $\mathcal{N\!N}(L, W)$ to the regularity of the discretized flow.
\end{remark}
Similarly to \cref{lem:dnn_lipschitz_in_parameters}, we find Lipschitz properties for the derivative $D \eta_\theta$ for fixed $\theta \in \ThetaLW$.

\begin{lemma}{(Lipschitz properties of $z \mapsto D\eta_\theta(z)$ on Bounded Input Sets)}\label{lem:lipschitz-Dxi}
Let $\Omega \subset \R^d$ be a bounded set, $L, W \in \N$ and $z \in \Omega$.
The mapping $D \eta_{(\cdot)}(z): \ThetaLW \to \R^W$ with $\theta \mapsto D\eta_\theta(z)$ is Lipschitz continuous with Lipschitz constant 
\begin{equation}
\widetilde{\Lambda}_{D\eta}^\Theta := L\left[(2W)^2 R_{L - 1}\right]^{L-1}
    \cdot(8 W^2 R_{L-1} + 2W^2 \widetilde\Lambda_\eta^\Theta + 2W(R_{L-1} + 1)) .
\end{equation}
Further, for fixed $\theta \in \Theta$ the mapping $z \mapsto D\xi_\theta(z,t)$
is Lipschitz continuous on $z \in \Omega$ with Lipschitz constant 
\begin{equation}
    \widetilde{\Lambda}_{D \eta}^{\Omega} := L \left[ (2W)^2 R_{L-1} \right]^{L-1} W^2 \Lambda.
\end{equation}
\end{lemma}

\begin{proof}
Concerning the parameter Lipschitz property, let $\theta = (\theta_1, \ldots, \theta_L) \in \Theta$ and $z \in \R^d$, then the derivative of the layer function \eqref{eq:layer-function} is $D \lambda_{\theta_\ell}^{\mathtt{ReQU}}(z) = 2 \cdot \mathrm{diag}((w_\ell z + b_\ell)\vee 0) \cdot w_\ell$.
For $\ell \in \{1, \ldots, L\}$, we define the subnetwork $\eta_\theta^{\ell}(z) := \lambda_{\theta_\ell}^{\mathtt{ReQU}} \circ \dots \circ \lambda_{\theta_1}^{\mathtt{ReQU}}(z)$ 
while we set $\eta_\theta^{0}(z) := z$.
Now, let $z_0 \in \Omega$ and $R_\ell$ as in \cref{lem:dnn_lipschitz_in_parameters}.
Then,
\begin{align}
    \begin{split}\label{eq:D(dnn)_norm_bound}
        \| D \lambda_{\theta_\ell}^{\mathtt{ReQU}} (\eta_\theta^{\ell - 1}(z_0))\|_2
        \leq& 2 \cdot \left\| w_\ell \eta_\theta^{\ell - 1}(z_0) + b_\ell \right\|_2 \cdot \|w_\ell\|_2 \\
        \leq& 2(W R_{\ell -1} + W) W \leq (2W)^2 R_{\ell - 1}.
    \end{split}
\end{align}
Let $\theta' \in \Theta$, then by the product rule,
\begin{align}
    \begin{split}
    \| D\eta_\theta(z) - D \eta_{\theta'}(z)\|_2 
    =& \left\| \prod_{\ell = 1}^L D \lambda_{\theta_\ell}^{\mathtt{ReQU}} (\eta_\theta^{\ell - 1}(z)) - \prod_{\ell = 1}^L  D \lambda_{\theta_\ell'}^{\mathtt{ReQU}} (\eta_{\theta'}^{\ell - 1}(z)) \right\|_2 \\
    \leq & \sum_{k = 1}^L \bigg\| \prod_{\ell = 1}^k D\lambda_{\theta_\ell}^{\mathtt{ReQU}} (\eta_\theta^{\ell - 1}(z)) \prod_{\ell = k + 1}^L D \lambda_{\theta_\ell'}^{\mathtt{ReQU}} (\eta_{\theta'}^{\ell - 1}(z)) \\
    &- \prod_{\ell = 1}^{k - 1} D \lambda_{\theta_\ell}^{\mathtt{ReQU}} (\eta_\theta^{\ell - 1}(z)) \prod_{\ell = k}^L D\lambda_{\theta_\ell'}^{\mathtt{ReQU}} (\eta_{\theta'}^{\ell - 1}(z)) \bigg\|_2 \\
    \leq& \sum_{k = 1}^L \left[\prod_{\ell = 1}^{k - 1} \left\|D\lambda_{\theta_\ell}^{\mathtt{ReQU}} (\eta_\theta^{\ell - 1}(z))\right\|_2\right]\cdot \left[\prod_{\ell = k + 1}^L \left\|D\lambda_{\theta_\ell'}^{\mathtt{ReQU}} (\eta_{\theta'}^{\ell - 1}(z))\right\|_2\right] \\
    &\cdot\left\|D \lambda_{\theta_k}^{\mathtt{ReQU}} (\eta_\theta^{k - 1}(z)) - D\lambda_{\theta_k'}^{\mathtt{ReQU}} (\eta_{\theta'}^{k - 1}(z))\right\|_2.
    \end{split}
\end{align}
Note, that by insertion of $\mathrm{diag}((w_k \eta_\theta^{k-1}(z) + b_k)\vee 0) w_k'$, the triangle inequality yields
\begin{align}
    \begin{split}        
    \big\|D \lambda_{\theta_k}^{\mathtt{ReQU}} (\eta_\theta^{k - 1}(z)) - &D\lambda_{\theta_k'}^{\mathtt{ReQU}} (\eta_{\theta'}^{k - 1}(z))\big\|_2
    \leq 2 \left\| \mathrm{diag}((w_k \eta_\theta^{k - 1}(z) + b_k)\vee 0)\right\|_2 \cdot \|\theta - \theta'\|_2 \\
    &+ 2\left\| w_k \eta_\theta^{k-1}(z) + b_k - w_k' \eta_{\theta'}^{k-1}(z) - b_k'\right\|_2 W \\
    \leq& 2 (2W)^2 R_{k -1} \|\theta - \theta'\|_2 + 2 W\|w_k\|_2 \|\eta_\theta^{k-1}(z) - \eta_{\theta'}^{k-1}(z)\|_2 \\
    &+ 2 W\|w_k - w_k'\|_2 \|\eta_{\theta'}^{k-1}(z)\|_2 + 2W \|b_k - b_k'\|_2 \\
    \leq& 2 (2W)^2 R_{k -1} \|\theta - \theta'\|_2 + 2 W^2 \widetilde\Lambda_\eta^\Theta \|\theta- \theta'\|_2 + 2 W \|\theta - \theta'\|_2 R_{k - 1} \\
    &+ 2W \|\theta - \theta'\|_2.
    \end{split}
\end{align}
Together with \cref{eq:D(dnn)_norm_bound}, we thus have
\begin{align}
    \begin{split}
    \|D \eta_\theta(z) - D \eta_{\theta'}(z)\|_2
    \leq& L\left[(2W)^2 R_{L - 1}\right]^{L-1} \left\|D \eta_\theta^k(\eta_\theta^{k - 1}(z)) - D\eta_{\theta'}^k(\eta_{\theta'}^{k - 1}(z))\right\|_2 \\
    \leq& L\left[(2W)^2 R_{L - 1}\right]^{L-1}\\
    \cdot&(8 W^2 R_{L-1} + 2W^2 \widetilde\Lambda_\eta^\Theta + 2W(R_{L-1} + 1)) \|\theta - \theta'\|_2.
    \end{split}
\end{align}
Similarly, for $z, \widetilde{z} \in \Omega$, we have
\begin{align}
    \begin{split}
    \|D\eta_\theta(z) - D \eta_\theta(\widetilde{z})\|_2
    \leq& \sum_{k=1}^L \left[ \prod_{\ell=1}^{k-1}\|D \lambda_{\theta_\ell}^{\mathtt{ReQU}}(\eta_\theta^{\ell - 1}(z))\|_2 \right] \cdot \left[ \prod_{\ell=k+1}^L \|D \lambda_{\theta_\ell}^{\mathtt{ReQU}}(\eta_\theta^{\ell - 1}(\widetilde{z}))\|_2 \right] \\
    &\cdot \| \lambda_{\theta_k}^{\mathtt{ReQU}}(\eta_\theta^{k-1}(z)) - \lambda_{\theta_k}^{\mathtt{ReQU}}(\eta_\theta^{k-1}(\widetilde{z}))\|_2 \\
    \leq& L \left[ (2W)^2 R_{L-1} \right]^{L-1} W^2 \Lambda_\eta^{\Omega} \|z - \widetilde{z}\|_2,
    \end{split}
\end{align}
where we used that
\begin{align}
    \begin{split}
    \|D \lambda_{\theta_\ell}^{\mathtt{ReQU}}(\eta_\theta^{k-1}(z)) &- D \lambda_{\theta_\ell}^{\mathtt{ReQU}}(\eta_\theta^{k-1}(\widetilde{z}))\|_2\\
    \leq& \| (w_\ell \eta_\theta^{k-1}(z)+ b_\ell \vee 0) - (w_\ell \eta_\theta^{k-1}(\widetilde{z}) + b_\ell \vee 0)\|_2 \cdot \|w_\ell\|_2\\
    \leq& \|w_\ell(\eta_\theta^{k-1}(z) - \eta_\theta^{k-1}(\widetilde{z}))\|_2 W 
    \leq W^2 \Lambda_\eta^{\Omega} \|z - \widetilde{z}\|_2.
    \end{split}
\end{align}
\end{proof}

\begin{remark}
By $L < 2^L$, we have that $\widetilde{\Lambda}_{D \eta}^{\mathcal{H}} \leq \Lambda^2$.
Particularly, on $\Omega = \Omega_d \times [0,1]$, we have the Lipschitz bound $\widetilde\Lambda_{D \eta}^\Theta\leq \Lambda^4$.
\end{remark}
\begin{lemma}{(Lipschitz Continuity of $z \mapsto \Psi_t^{2, h}(z;\eta_\theta)$ on Bounded Input Sets)}\label{lem:rk-lipschitz-z}
    Let $\Omega \subset \R^d$ be a bounded set, $h > 0$, $t \geq 0$ and $\theta \in \Theta$. 
    The Runge-Kutta step $z \mapsto \Psi_t^{2, h}(z;\eta_\theta)$ is Lipschitz-continuous for $z \in \Omega$ with Lipschitz constant $\Lambda_\Psi^\Omega(h) := 1 + h \widetilde\Lambda_\eta^\Omega + \tfrac{1}{2} h^2 (\widetilde\Lambda_\eta^\Omega)^2$.
\end{lemma}
\begin{proof}
    Let $x,y \in \Omega$, then by \cref{cor:dnn-lipschitz-z}, we have
    \begin{align}
        \begin{split}
        \|\Psi_t^{2,h}(x; \eta_\theta) - \Psi_t^{2,h}(y;\eta_\theta)\|_2
        =& \| x + \psi_t^h(x, \eta_\theta) - y - \psi_t^h(y, \eta_\theta)\| \\
        \leq& \left(1 + h \widetilde\Lambda_\eta^\Omega + \frac{h^2}{2} (\widetilde\Lambda_\eta^\Omega)^2\right) \|x - y\|_2. 
        \end{split}
    \end{align}
\end{proof}
The $h$-dependent Lipschitz constant $\Lambda_\Psi^\Omega(h)$ found above can be further conventiently upper bounded by $\e^{h \Lambda}$.
\begin{lemma}{(Lipschitz Continuity of $\theta \mapsto \Psi_t^{2, h}(z;\eta_\theta) $ on Bounded Input Sets)}\label{lem:rk-lipschitz-theta}
    Let $\Omega \subset \R^d$ be a bounded set, $z \in \Omega$, $h > 0$ and $t \geq 0$.
    The map $\Psi_t^{2, h}(z;\eta_{(\cdot)}): \Theta \to \R^W$ with $\theta \mapsto \Psi_t^{2, h}(z;\eta_\theta)$ is Lipschitz continuous with constant $h \cdot \Lambda_\Psi^\Theta(h)$, where $\Lambda_\Psi^\Theta(h) := \widetilde\Lambda_\eta^\Theta (1 + \tfrac{h}{2} \Lambda_\eta^{\Omega})$.
\end{lemma}
\begin{proof}
    Let $\theta, \theta' \in \Theta$. Then, by \cref{cor:dnn-lipschitz-z,lem:dnn_lipschitz_in_parameters}, we have
    \begin{align}
        \begin{split}
        \|\Psi_t^{2,h}(z; \eta_\theta) - \Psi_t^{2,h}(z;\eta_{\theta'})\|_2
        \leq& \| \psi_t^h(x; \eta_\theta) - \psi_t^h(x; \eta_{\theta'})\|_2 \\
        \leq& h \widetilde\Lambda_\eta^{\Omega} \frac{h}{2} \|\eta_\theta(z,t) - \eta_{\theta'}(z,t)\|_2 + h \widetilde\Lambda_\eta^\Theta\|\theta - \theta'\|_2 \\
        \leq& h \widetilde\Lambda_\eta^\Theta \left(1 + \frac{h}{2} \widetilde\Lambda_\eta^{\Omega}\right) \| \theta - \theta'\|_2.
        \end{split}
    \end{align}
\end{proof}
\begin{remark}
    Assuming that \eqref{eq:assumed-h-bound} holds, we also have that $\Lambda_\Psi^\Theta(h) \leq 2 \Lambda^2$ and $h \Lambda_{\mathrm{RK}}^\Theta(h) \leq \tfrac{1}{4} \Lambda (1 + \tfrac{1}{8}) < \Lambda$.
\end{remark}

\begin{lemma}{(Lipschitz Continuity of $\theta \mapsto\Psi_\theta^h$ on Bounded Input Domains)}\label{lem:lipschitz-flow}
Let $\Omega \subset \R^d$ be a bounded set, $z \in \Omega$, $m \in \N$ and $h:= 1/m$.
The map $\Psi_{(\cdot)}^h(z): \Theta \to \R^W$ with $\theta \mapsto \Psi_\theta^h(z)$ is Lipschitz continuous with Lipschitz constant $\LamRK^\Theta := \Lambda_\Psi^\Theta \cdot \exp(\widetilde\Lambda_\eta^{\Omega})$.
\end{lemma}
\begin{proof}
Let $\theta, \theta' \in \Theta$ and for $k \in \N$ let $z_{kh}' := \Psi_{0, kh}^{2,h}(z; \eta_{\theta'})$.
Then, similarly to the proof of \cref{lem:lipschitz-Dxi} and using \cref{lem:rk-lipschitz-z,lem:rk-lipschitz-theta} we have
\begin{align}
    \begin{split}
    \|\Psi_\theta^{h}(z) - \Psi_{\theta'}^h(z)\|_2 
    \leq& \sum_{k=0}^{m-1} \bigg\| \left[(\id + \psi_{1-h}^\theta)\circ \cdots \circ (\id + \psi_{(k+1)h}^\theta) \circ (\id + \psi_{kh}^\theta)\right] (z_{kh}') \\
    & - \left[(\id + \psi_{1 - h}^\theta) \circ \cdots \circ (\id + \psi_{(k+1)h}^\theta) \circ (\id + \psi_{kh}^{\theta'})\right] (z_{kh}')\bigg\|_2 \\
    \leq& \sum_{k = 0}^{m-1} \left[\Lambda_\Psi^\Omega\right]^{m - k - 1} \| \psi_{kh}^{h}(z_{kh}'; \eta_{\theta'}) - \psi_{kh}^{h}(z_{kh}'; \eta_\theta)\|_2 \\
    \leq& \sum_{k=0}^{m-1} \left[\Lambda_\Psi^\Omega\right]^{m - k - 1} \cdot h \Lambda_{\Psi}^\Theta\|\theta - \theta'\|_2 
    \leq \e^{\widetilde\Lambda_\eta^{\Omega}} \Lambda_\Psi^\Theta \|\theta - \theta'\|_2.
    \end{split}
\end{align}
\end{proof}
We used the coarser estimate $\Lambda_\Psi^\Omega \leq \e^{h \widetilde\Lambda_\eta^\Omega}$ since the leading order term in $h = 1/m$ is 
\begin{equation}
    [\Lambda_\Psi^\Omega]^{m} = (1 + h \widetilde\Lambda_\eta^\Omega + (h\widetilde\Lambda_\eta^\Omega)^2)^{\tfrac{1}{h}} \sim \e^{ \widetilde\Lambda_\eta^\Omega}, \qquad (h\widetilde\Lambda_\eta^\Omega \to 0).
\end{equation}
We, therefore, cannot expect any better asymptotic behavior from the estimates obtained so-far together with the standard step-wise error bounding method employed above.
Treatment of the numerics in this way leads to a Lipschitz constant that is exponential in $\Lambda$.

\subsection{Learning Objective}
The objective in statistical learning is minimizing some distance function $\mathrm{d}(\mu\| \widehat\mu)$ with respect to some model $\widehat\mu$.
One of the most widely used is the Kullback-Leibler divergence between $\mu$ and $\widehat\mu$, given by
\begin{equation}
\mathrm{KL}(\mu\|\widehat\mu) = \int_{\Omega} \log \left( \frac{\mathrm{d}\mu}{\mathrm{d} \widehat\mu} \right) \mathrm{d}\mu.
\end{equation}
Assuming Lebesgue densities $f_{\mu}$ and $f_{\widehat\mu}$ of $\mu$ and $\widehat\mu$, respectively, minimization of $\mathrm{KL}(\mu\|\widehat\mu)$ is equivalent to minimization of the expected negative log-likelihood
\begin{equation}
- \int_{\Omega} \log \left( f_{\widehat\mu}(x) \right) \cdot f_{\mu}(x) \mathrm{d}x
\end{equation}
which requires evaluation of the log density of the model $\widehat\mu$.
The latter is notoriously difficult to evaluate for generative models since the density of the parametric model $\widehat\mu = \Phi^\theta_*\nu$ is by change of variables $f_{\widehat\mu}(x) = |\det D_{\! x} (\Phi^{\theta})^{-1}(x)|$ which inspired a range of specialized design choices such as affine coupling flows.
In the case of neural ODEs, the log-determinant of the flow takes on a reasonably simple form discussed in \cref{sec:liouville-formula}.
This tells us, that we only require the divergence of $\eta_{\theta}$ instead of the entire Jacobian of $\Phi^{\theta}$.
For simplicity, we will in the following reverse the roles of $\Phi^\theta$ and $(\Phi^\theta)^{-1}$ and regard the computation of $|\det D_x \Phi^\theta(x)|$.

\subsection{Statistical Learning}
Stochastic guarantees of approximating measures $\mu \in \mathcal{T}$ in some target set $\mathcal{T} \subset \mathcal{M}_{+}^{1}(\Omega)$ by a model $\widehat\mu$ can be given in terms of statistical learning from samples $X_{1}, \ldots, X_{n} \sim \mu$.
We call a family of maps $\left\{ \widehat{\mu}_{n} \right\}_{n \in \mathbb{N}}$ where $\widehat{\mu}_{n}: \Omega^{n} \to \mathcal{H}_{n} \subset \mathcal{M}_{+}^{1}(\Omega)$ a statistical learning algorithm over the (sequence of) hypothesis spaces $\{\mathcal{H}_{n}\}_{n \in \mathbb{N}}$.
The learner $\widehat{\mu}_{n}$ maps $n$ data samples $(x_{1}, \ldots, x_{n}) \in \Omega^{n}$ to a probability distribution with the target of being close to $\mu$ in the sense that $\mathrm{KL}(\mu\|\widehat{\mu}_{n}(x_{1}, \ldots, x_{n}))$ (or some other notion of distance on $\mathcal{M}_+^1(\Omega)$) becomes small.
From a probabilistic point of view, we regard the data points as i.i.d. random variables $X_{1}, \ldots, X_{n} \sim \mu$ and collect them in $\chi_{n} := (X_{1}, \ldots, X_{n})$.
The target space $\mathcal{T}$ is called learnable by $\left\{ \widehat{\mu}_{n} \right\}_{n \in \mathbb{N}}$ with respect to $\mathrm{KL}$, if for each $\mu \in \mathcal{T}$,
\begin{equation}
 \mathrm{KL}(\mu\|\widehat{\mu}_{n}(\chi_{n})) \stackrel{\mu}{\to } 0 \qquad (n \to \infty).
\end{equation}
$\mathcal{T}$ is probably approximately correctly (PAC) learnable by $\left\{ \widehat{\mu}_{n} \right\}_{n \in \mathbb{N}}$ if there is a function $n: (0,1)^{2} \to \mathbb{N}$ such that 
\begin{equation}
\Pr \left( \mathrm{KL}(\mu\|\widehat{\mu}_{n}(\chi_{n})) > \varepsilon \right) \leq \delta
\end{equation}
for all $n \geq n(\varepsilon, \delta)$.
That is, there is a stochastic guarantee that given $n(\varepsilon, \delta)$ points of data sampled from $\mu$, $\widehat{\mu}_{n(\varepsilon, \delta)}$ approximates $\mu$ with error more than $\varepsilon$ (precision) with probability no more than $\delta$ (confidence).

Practically, minimizing $\mathrm{KL}(\mu\| \widehat{\mu}_{n})$ or the expected negative log-likelihood is impossible since $\mu$ is not known. 
Therefore, the samples $\chi_{n}$ are utilized for formulating a proxy optimization objective, called an empirical risk function. 
That is, a sequence of triples $(c_{n}, e_{n}, \widehat{L}_{n})$ such that $c_{n} > 0$, $h_{n}: \mathcal{T} \to \mathbb{R}$ and $\widehat{L}_{n}: \mathcal{H}_{n} \times \Omega^{n} \to \mathbb{R}$ with the requirement that 
\begin{equation}
c_{n_{k}} \cdot \widehat{L}_{n_{k}}(\nu, \chi_{n_{k}}) + e_{n_{k}}(\mu) \stackrel{\mu}{\to } \mathrm{KL}(\mu\|\nu), \qquad (k \to \infty)
\end{equation}
for all sub-sequences $\left\{ n_{k} \right\}_{k \in \mathbb{N}}$, any $\mu \in \mathcal{T}$, and any $\nu \in \mathcal{M}_{+}^{1}(\Omega)$ such that also $\nu \in \mathcal{H}_{n_{k}}$ for all $k \in \mathbb{N}$.
Again, here $\chi_{n} = (X_{1}, \ldots, X_{n})$ with $X_{1}, \ldots, X_{n} \sim \mu$ i.i.d.

In the following, consider dimension $d \in \mathbb{N}$, 
\begin{equation}
    \Omega_{d}:= B_{1/2}(0) = \{y \in \R^d: \|y\|_2 < \tfrac{1}{2}\},
\end{equation}
the open Euclidean ball of radius $1/2$ centered at the origin in $\R^d$ and regularity $C^k$ for $k \geq 4$ and let $\nu_0 := \mathcal{U}(\Omega_d)$ be the uniform distribution on $\Omega_d$.
The target space we aim at learning with discretized neural flows is
\begin{equation}
\mathcal{T}_{d,k} = \left\{ \mu \in \mathcal{M}_+^1(\Omega_d)\left| \exists \alpha \in (0,1]:f_\mu := \tfrac{\mathrm{d}\mu}{\mathrm{d} \nu_0}\in C^{k,\alpha}(\Omega_d, \R),\,  f_\mu > 0 \right.\right\}.
\end{equation}
By \cref{theorem:Ck_alphaForV}, we then know that for any $\mu \in \mathcal{T}_{d,k}$ there is $\alpha \in (0, 1]$ and a unique vector field $\xi \in C^{k,\alpha}(\Omega_d, \R^d)$ which generates a $C^{k,\alpha}$-flow $\Phi_t$ that transports $\mu$ to $\nu_0$.
Our choice of hypothesis spaces may, therefore, be guided by our effort to approximate $\xi$ by a suitably regular neural network.

The hypothesis spaces we consider a priori for learning $\mathcal{T}_{d,k}$ are the transports 
\begin{equation}
\mathcal{H}_{L, W, K}^h := \left\{\widehat{\mu} \in \mathcal{M}_+^1(\Omega_d): \mu = \left(\Psi^h(\eta_\theta)\right)^{-1}_* \nu_0, \eta_\theta \in \mathcal{N\! N}(L, W), h \in (0,1), K \in \mathbb{N}\right\}.
\end{equation}
under discretized neural flows generated from $\mathtt{ReQU}$-neural networks \linebreak $\mathcal{N\! N}(L, )$ of width $W$, depth $L$ and weights $\theta \in \ThetaLW$.
Here, $h \in (0,1)$ is the discretization time step utilized for second-order Runge-Kutta integration which becomes an architecture hyperparameter.
The parameter $K$ is coupled to the width $W$ and the regularity parameter $k$ through a universal approximation result we will utilize in \cref{sec:model-error}.
The choice of letting a learning algorithm $\{\widehat{\mu}_n\}_{n \in \mathbb{N}}$ be defined on $n$-dependent hypothesis spaces $\{\mathcal{H}_n\}_{n \in \mathbb{N}}$ lets us adapt the hyperparameters $h$, $L$, $W$ and $K$ to $n$.
Keeping this in mind, we will in the following use the short-hand $\mathcal{H}_n = \mathcal{H}_{L_n, W_n, K_n}^{h_n}$ in the concrete neural network setting defined above.

\subsection{Empirical Risk Minimization}
This formulation comes with the advantage that the optimization objective for some $\nu$ is split off the explicit contribution of the distribution $\mu$ with influence only by the samples $\chi_{n}$ drawn from $\mu$.
One can, therefore, formulate the proxy objective of empirical risk minimization seeking
\begin{equation}
\widehat{\mu}_{n}(\chi_{n}) \in \mathrm{argmin}_{\nu \in \mathcal{H}_{n}}\widehat{L}_{n}(\nu, \chi_{n}).
\end{equation}
If both, $\mu$ and $\nu$ have Lebesgue densities, an unbiased empirical risk functional for the KL divergence is the negative log-likelihood
\begin{equation}\label{eq:log-likelihood}
    \widehat{L}_{n}(\nu, \chi_{n}) = - \sum_{i = 1}^{n} \log(f_{\nu}(X_{i}))
\end{equation}
with $c_{n} = 1/n$ and $e_{n}(\mu) = e(\mu) = \mathbb{E}_{X \sim \mu}[\log(f_{\mu}(X))]$ the entropy of $\mu$.
Empirical risk minimization introduces an error in the optimization objective which is governed by the stochastic behavior of $\mu$ which can be made explicit by the ERM error decomposition.
Let $\nu' \in \mathcal{H}_{n}$ and let
\begin{equation}
\varepsilon^{\mathrm{model}}_{n} := \inf_{\widetilde \nu \in \mathcal{H}_{n}} \mathrm{KL}(\mu\|\widetilde{\nu}), \quad
\varepsilon^{\mathrm{gen}}_{n}(\chi_n) := \sup_{\widetilde{\nu} \in \mathcal{H}_{n}} |\mathbb{E}_{X \sim \mu}[- \log (f_{\widetilde{\nu}} (X))] - \widehat{L}_{n}(\widetilde{\nu}, \chi_{n})|.
\end{equation}
Then, we have the following bound
\begin{align}
    \begin{split}
    \mathrm{KL}(\mu\|\nu) =& \mathbb{E}_{X \sim \mu}[- \log(f_{\nu}(X))] + e(\mu)
\leq \widehat{L}_{n}(\nu, \chi_{n}) + \varepsilon_{n}^{\mathrm{gen}} + e(\mu) \\
\leq& \mathbb{E}_{X \sim \mu}[- \log(f_{\nu'}(X))] + 2 \varepsilon_{n}^{\mathrm{gen}} + e(\mu)
= \mathrm{KL}(\mu\| \nu') + 2 \varepsilon_{n}^{\mathrm{gen}}
    \end{split}
\end{align}
and in particular, also $\mathrm{KL}(\mu\|\nu) \leq \varepsilon_{n}^{\mathrm{model}} + 2 \varepsilon_{n}^{\mathrm{gen}}$.
In the following sections, we will derive a universal approximation bound for the model error $\epsilon_{n}^{\mathrm{model}}$ and concentration bounds for the generalization error $\epsilon_{n}^{\mathrm{gen}}$ for generative learning with neural ODEs.
The obtained bound on $\mathrm{KL}(\mu\|\nu)$ will further yield learnability and PAC learnability.

\section{Main Results}
Our main results employ a deterministic model error bound based on a universal approximation result for Hölder functions and probabilistic generalization error bounds obtained by concentration techniques.
\subsection{Model Error Bounds}
\label{sec:model-error}
Let $\mu = (\Phi(\xi))^{-1}_* \nu \in \mathcal{T}_{d,k}$, then the change of variables formula guarantees that $\mu$ has a $\nu_0$-density given by $f_\mu(x) = |\det(D_x \Phi(x;\xi))|$, so in particular we have
\begin{align}
    \epsilon^\mathrm{model}_n = 
    \inf_{\nu \in \mathcal{H}_n} \mathrm{KL}\left(\mu \| \nu\right)
    = \inf_{(\Psi^{h}(\eta_\theta))^{-1}_* \nu_0 \, \in \mathcal{H}_n} \int_{\mathbb{R}^d} \frac{\log(|\det(D\Phi((x;\xi))|)}{\log(|\det(D \Psi^{h}(x;\eta_\theta)|)} \cdot f_\mu(x) \, \mathrm{d}x.
\end{align}
In the following, we proceed by showing that the log-determinant terms can be approximated by the flows generated by $\mathtt{ReQU}$ neural networks $\eta_\theta$.
We use a fact first proven by Belomestny et al.~\cite[Thm.\ 2]{simultaneous_approx} stating that $(k, \alpha)$-Hölder-continuous functions can be $C^\ell$-approximated by \texttt{ReQU} neural networks and apply it to approximation of flow endpoints of the ``true'' vector field $\xi$ by flow endpoints of the approximated vector field $\eta_\theta$.

\begin{theorem}{(Approximation of Flow Endpoints by Neural Flow Endpoints)}\\
    \label{theorem:approx_neuralflow}
    For every flow endpoint $\Phi = \Phi_{0,1}$ generated by a compactly $y$-supported vector field $\xi \in C^k(\R^d \times \R, \R^d)$ with $y$-support in $\Omega_d$, $k\geq 4$, and every $K \in \N, K > 2(k+1)$ there exists a neural flow endpoint $\Phi_\theta = \Phi^\theta_{0,1}$ generated by a \texttt{ReQU}-network $\eta_\theta \in C^{k-1}(\R^d \times \R)$ with $y$-support in $\Omega_d$ and
\begin{align}
\left| \log \frac{\mathrm{d} \mu}{\mathrm{d} \mu_\theta} \right| := 
\left|\log(|\det(D_y\Phi(y))|) - \log(|\det(D_y\Phi_\theta(y))|)   \right| \leq C_\xi \frac{1}{K^{k-1}},
\end{align}
where $C_\xi$ is a constant depending on $\xi$ and $ k,d$.
The network $\eta_\theta$ consists of a neural network $\widetilde{\eta}_\theta$ with  depth 
\begin{equation}
 6 + 2((k-1)-2)+ \lceil \log_2(d+1) \rceil + 2 (  \lceil  \log_2((d+1)(2k-1)) +  \log_2(\log_2(\|\xi\|_{C^k}))  \rceil+1), 
\end{equation}
width $\lceil 12(d+1)(3K)^{d+1} \rceil$
and weights in $[-1, 1]$, which is multiplied by a cutoff neural network delivering the appropriate degree of regularity on $\partial \Omega_d$.
\end{theorem}
\begin{proof}
We will approximate $\overline{\xi}(y,t):= \xi(y - \tfrac{1}{2} \vec{1}_d, t)$ in $[0,1]^{d+1}$ for which we introduce a constant shift $\tfrac{1}{2} \vec{1}_d$ of length $\sqrt{d}/2$ along the $y$-diagonal such that the restriction of $\overline{\xi}$ to $[0,1]^{d+1}$ is $C^{k-1, 1}$ and, therefore, allows by Theorem 2 in \cite{simultaneous_approx} for the existence of a \texttt{ReQU} approximation $\widetilde{\eta}_\theta$ with
\begin{align}
    \label{C ell bound}
    \|\overline{\xi} - \widetilde{\eta}_\theta\|_{C^\ell([0,1]^{d+1})} \leq \frac{(1+9^{(d+1)(k-2)}(2k-1)^{2d+5})(\sqrt{2}e(d+1))^{k}2\|\xi\|_{C^{k}}}{K^{k-\ell}} =: \frac{c(d, k, \xi)}{K^{k - \ell}}
\end{align}
for $\ell = 0,\dots,3$.
The neural network $\widetilde{\xi}_\theta$ is $C^{k-1}$ on $[0,1]^{d+1}$ with weights in $[-1, 1]$. 
However, we require it to be $C^{k-1}$ on $\R^{d} \times [0,1]$ with $y$-support in $\Omega_d$.
We implement a spatial shift of $\widetilde{\eta}_\theta$ back to $[-\tfrac{1}{2}, \tfrac{1}{2}]^d \times [0, 1]$ which can be implemented as biases in the first layer.
By construction of $\widetilde{\eta}_\theta$, this leaves the weights in $[-1, 1]$.
Further, we will multiply it by sums of B-splines to ensure the needed regularity on $\partial \Omega_d$.
By the support of $\xi$ and the mean value theorem, we have the bounds
\begin{align}
\label{eq:estimate_xi_norm}
    \left\| \xi(y,t) \right\|_2 \leq \Lip(D^m \xi) \cdot \min_{j=1,\dots,d} \min\{|- \tfrac{1}{2} - y_j|, |\tfrac{1}{2} -y_j|\}^{m+1}, 
\end{align}
and 
\begin{align}
    \label{eq:estimate_Dxi_norm}
    \left\| \partial_{y_\ell} \xi_i(y,t) \right\|_2 \leq \Lip(D^m \xi) \cdot \min_{j=1,\dots,d} \min\{|- \tfrac{1}{2} - y_j|, |\tfrac{1}{2}-y_j|\}^{m}, 
\end{align}
for $m \in \{0, \dots, k-1\}$.
We now define the cutoff spline $\chi_{K}(y) = \sum_{j=0}^{K-k-1} B_{k,j}(4\|y\|_2^2)$
to restrict the spatial support of the approximation $\eta_\theta(y,t) := \chi_K(y) \cdot \widetilde{\eta}_\theta(y,t)$ to $\Omega_d$.
The network $\eta_\theta$ requires one additional layer for the multiplication and computing $\chi_K$ additionally requires at most the same width as computing $\widetilde{\eta}_\theta$.
Note that computing $4 \|y\|_2^2 = 4\sum_{j = 1}^d y_j^2$ can be achieved in 2 layers with weights in $[-1, 1]$.
The flow $\Phi(\eta_\theta)$ generated by $\eta_\theta \in C^{k-1}(\R^d\times [0,1], \R^d)$ is the identity outside of $\Omega_d$.
We now obtain using the partition of unity property of $\chi_K$ and \cref{eq:estimate_xi_norm} that
\begin{align}
    \begin{split}
        \|\xi - \eta_\theta\|_{C^0(\R^d \times [0,1])} =& \sup_{y \in \Omega_d, t \in [0,1]} \|\xi(y,t) - \widetilde{\eta}_\theta(y + \tfrac{1}{2}\vec{1}_d, t) \chi_K(y)\|_2 \\
        \leq& \sup_{y \in \Omega_d, t \in [0,1]} \Big\{|\chi_K(y)| \cdot \|\xi(y,t) - \widetilde{\eta}_\theta(y + \tfrac{1}{2} \vec{1}_d, t)\|_2 \\
        &\hspace{3.5cm}+ |1 - \chi_K(y)| \cdot \|\xi(y,t)\|_2\Big\} \\
        \leq& \|\overline{\xi} - \widetilde{\eta}_\theta\|_{C^0(\R^d\times[0,1]} + \sup_{\|y\|_2 \in \left[\tfrac{K - k + 1}{2K}, 1/2\right], t \in [0,1]} \|\xi(y,t)\|_2 \\
        \leq& \|\overline{\xi} - \widetilde{\eta}_\theta\|_{C^0(\R^d\times[0,1]} + \Lip(D^m\xi) \cdot \frac{(k+1)^{m+1}}{K^{m+1}}
    \end{split}
\end{align}
for $m \in \{0, \dots, k-1\}$.
We shall require bounds on the partial derivative terms $|\partial_i \xi_i(y,t) - \partial_i [\eta_\theta]_i(y,t)|$ for which we obtain by triangle decomposition
\begin{align}
    \begin{split}
    |\partial_i \xi_i - \partial_i [\eta_\theta]_i| 
    \leq& |\chi_K|\cdot|\partial_i \overline\xi_i - \partial_i [\widetilde{\eta}_\theta]_i| 
    + |\partial_i \chi_K| \cdot |\overline\xi_i - [\widetilde{\eta}_\theta]_i|\\
    &+ |1 - \chi_K| \cdot |\partial_i \xi_i| + |\partial_i \chi_K| \cdot |\xi_i|     
    \end{split}
\end{align}
where we have omitted the arguments $y \in \Omega_d$ and $t \in [0,1]$ for brevity.
Bounding the derivatives of $\chi_K$ by Corollary 2 in \cite{simultaneous_approx}, we have
\begin{align}\label{eq:del_xi_difference_bound}
    \begin{split}
    \left\| \frac{\partial \xi_i}{\partial y_i} - \frac{\partial [\eta_\theta]_i}{\partial y_i} \right\|_{C^0(\R^d \times [0,1])}
    \leq& \|\overline\xi - \widetilde{\eta}_\theta\|_{C^1(\R^d \times [0,1])} + 2k^2 K  \|\overline\xi - \widetilde{\eta}_\theta\|_{C^0(\R^d \times [0,1])} \\
    &+ \Lip(D^{k-1} \xi) \frac{(k+1)^{k-1}}{K^{k-1}}+ 2k^2 \Lip(D^{k-1}\xi) \frac{(k+1)^k}{K^{k-1}}.
    \end{split}
\end{align}
Finally, by \cref{lem:logdet-formula,eq:del_xi_difference_bound}, we find

\begin{align}
    \begin{split}
    \left| \log \frac{\mathrm{d} \mu}{\mathrm{d} \mu_\theta} \right|
    &\leq  \sum_{i=1}^d \int_0^1 \left| \frac{\partial \xi_i}{\partial y_i}(\Phi_{0,t}(y),t) - \frac{\partial [\eta_\theta]_i}{\partial y_i}(\Phi^\theta_{0,t}(y),t) \right|\;dt  \\
    &\leq  \sum_{i=1}^d \int_0^1 \Lip(D\xi) \cdot \|\Phi_{t} - \Phi_{t}^\theta\|_\infty + \left| \partial_i\xi_i(\Phi_{t}^\theta(y),t)- \partial_i [\eta_\theta]_i(\Phi_{t}^\theta(y),t) \right| \;dt  \\
    &\quad+ \Lip(D^{k-1} \xi) \frac{(k+1)^{k-1}}{K^{k-1}}+ 2k^2  \Lip(D^{k-1}\xi) \cdot \frac{(k+1)^k}{K^{k-1}} \bigg)\\
    &\leq d\left( \frac{\Lip(D\xi)}{\Lip(\xi)} (\e^{\Lip(\xi)}-1) + 1 + 2k^2 \right) \frac{c(d, k, \xi)}{K^{k-1}}\\
    &\quad+ \Lip(D^{k-1} \xi) \frac{(k+1)^{k-1}}{K^{k-1}}+ 2k^2  \Lip(D^{k-1}\xi) \frac{(k+1)^k}{K^{k-1}} 
    =: C_\xi \frac{1}{K^{k-1}}.
    \end{split}
\end{align}
\end{proof}

\begin{remark}
    By fixing $\xi$, we see from our definition of hypothesis spaces $\mathcal{H}_{L, W, K}^h$, that we only need a fixed depth $L$ for arbitrarily accurate approximation.
    Choosing
    \begin{equation}
        \label{eq:bound-on-K}
        K = K(W) =  \left\lfloor \frac{1}{3} \left( \frac{1}{48(d+1)} W \right)^{\frac{1}{d+1}} \right\rfloor
    \end{equation}
    the width of the neural network $\widetilde{\eta}_\theta$ is sufficient for applying \cref{theorem:approx_neuralflow}.
    Moreover, our radial cutoff spline requires an additional layer to the depth of $\widetilde{\eta}_\theta$.
\end{remark}
Based on the Runge-Kutta approximation error bounds from \cref{theorem:discretized_ode-flow-bound}, we obtain the following bound on the log determinant difference:
\begin{theorem}{(Log-Determinant Bounds for Discretized Neural Flows)}\label{theorem:discretized_ode_logdet-bound}
    Let $\Phi_\theta$ be a neural flow with the generating neural network $\eta_\theta$, which is in $C^k$ on some compact set $\Omega$ and the identity outside of $\Omega$ for $k \geq 3$ and $\Psi^h(\eta_\theta)$ the corresponding discretized neural flow obtained by the Runge-Kutta scheme of order 2 with step size $h < \frac{1}{2 \Lip(\eta_\theta)}$. Then we have
    \begin{align}
        |\log|\det(D\Phi_\theta(y))| -\log|\det(D\Psi^h(y;\eta_\theta))||
        \leq \frac{1}{h} \left(\left( 2 \cdot h^3 C(\theta,d) +1\right)^d -1\right)
\end{align}
with 
\begin{align}
    \begin{split}
    C(\theta,d) &=   \frac{1}{24} \Gamma_{\theta,2} + \frac{d}{8}\Gamma_{\theta,1}\|\eta_\theta\|_{C^1} +  \|\eta_\theta \|_{C^2}  \frac{d C_\theta}{4}  \left(1 + \frac{hd}{2}\|\eta_\theta\|_{C^1}\right) \\
    &\quad +    \left(\left(1+\frac{h}{2}d\|\eta_\theta\|_{C^1} \right)^2 
      + \frac{h}{2} d\|\eta_\theta\|_{C^1}\right) C_\theta  (\e^{ \|\eta_\theta \|_{C^2} }-1)
    \end{split}
\end{align}
and $C_\theta = C(\eta_\theta)$.
Whenever not explicitly stated otherwise, we denote $\|\cdot\|_{C^k} := \| \cdot\|_{C^k(\R^d \times [0,1]; \R^d)}$.
\begin{proof}
    Since $\Phi^\theta$ and $\Psi_\theta^h$ are the identity outside of $\Omega$, we only regard $y \in \Omega$.
    By the chain rule and equation \eqref{eq:runge-kutta-composition}, we have
\begin{align}
    \begin{split}
    \log \frac{|\det \,D\Phi_\theta(y)|}{|\det \, D\Psi^h_\theta(y)|}
    &= \sum_{k=0}^{m-1}
    \log \frac{|\det \, D\Phi_{kh,(k+1)h}^\theta(\Phi_{0,kh}^\theta(y))|}{| \det \, D (\mathrm{id} + \psi_{kh}^h(\cdot; \eta_\theta))(\Psi_{\theta}^h(y))|}
    \end{split}
\end{align}
From $\Lip(\eta_\theta) < \tfrac{1}{2h}$, we have that
\begin{equation}
    \|D \psi_{kh}^h(\Psi_{\theta}^h(y);\eta_\theta)\|_2 \leq h \Lip(\eta_\theta) (1 + \tfrac{h}{2} \Lip(\eta_\theta)) < \tfrac{1}{2}.
\end{equation}
Hence, $D (\mathrm{id} + \psi_{kh}^h(\cdot; \eta_\theta)) = D \Psi_{kh}^h(\cdot; \eta_\theta)$ is non-singular.
By the elementary bound $\log(a/b) \leq (a - b)/b$ for $a, b > 0$, the matrix perturbation bound in Cor.\ 2.14 in \cite{determinants} and the Neumann series bound in Thm.\ V.6.1 in \cite{wernerEinfuehrungHoehereAnalysis2009a}, we have with $z:=\Phi_{0,kh}^\theta(y)$ and $z^h:=\Psi_\theta^h(y)$ that
\begin{equation}
     \log \frac{|\det \, D\Phi_{kh,(k+1)h}^\theta(z)|}{| \det \, D \Psi_{kh}^h(\cdot; \eta_\theta)(z^h)|}
     \leq (2 \cdot \| D \Phi_{kh, (k+1)h}^\theta (z) - D \Psi_{kh}^h(\cdot; \eta_\theta)(z^h)\|_2 + 1)^d - 1.
\end{equation}
for $k = 0, \ldots, m-1$.
Treating the spectral norm term, we find estimates for the triangle decomposition
\begin{align}
\label{eq:spectral-diff-decomposition}
\begin{split}
    \|D \Phi_{kh,(k+1)h}^\theta(z) - D \Psi_{kh}^h(z^h;\eta_\theta)\|_2
    \leq& \| D \Phi_{kh,(k+1)h}^\theta(z) - D\Psi_{kh}^h(z;\eta_\theta)\|_2 \\
    &+ \|D\Psi_{kh}^h(z;\eta_\theta) - D \Psi_{kh}^h(z^h;\eta_\theta)\|_2.
\end{split}
\end{align}
The first term in \eqref{eq:spectral-diff-decomposition} can be further decomposed as
\begin{align}
    \label{eq:sdc-first-term-decomposition}
    \begin{split}
    \|D\Phi&_{kh, (k+1)h}^\theta(z) - D\Psi^h_{kh}(z;\eta_\theta)\|_2   \\
    &\leq  \bigg\|\int_{kh}^{(k+1)h} D\eta_\theta(\Phi^\theta_{kh,\tau}(z), \tau)D\Phi^\theta_{kh,\tau}(z) \; d\tau \\
    &\quad- hD\eta_\theta \left(\Phi^\theta_{kh, kh+\frac{h}{2}}(z),kh+\tfrac{h}{2}\right)D\Phi^\theta_{kh,kh+\frac{h}{2}}(z)
    \bigg\|_\mathrm{F} \\
    &\quad + h\left\|D\eta_\theta\left(\Phi^\theta_{kh, kh+\tfrac{h}{2}}(z),kh+\tfrac{h}{2} \right) \right\|_\mathrm{F}  \left\|D\Phi^\theta_{kh,kh+\tfrac{h}{2}}(z) - \left(I + \tfrac{h}{2}D\eta_\theta(z,kh)\right) 
    \right\|_\mathrm{F} \\
    &\quad + h\left\| D\eta_\theta \left(\Phi^\theta_{kh, kh+\tfrac{h}{2}}(z),kh+\tfrac{h}{2} \right)-D\eta_\theta\left( z+\tfrac{h}{2}\eta_\theta(z,kh),kh+\tfrac{h}{2}\right) 
    \right\|_2 \\
    &\qquad\cdot \left\| I + \tfrac{h}{2}D\eta_\theta(z,kh) \right\|_2\\
    \end{split}
\end{align}
Here, we utilize the midpoint rule~\cite[Thm.\ 8.41]{RichterThomas2017EidN} and the left rectangle rule~\cite[Thm.\ 8.40]{RichterThomas2017EidN} to estimate the first and second term against expressions 
\begin{equation}
\label{eq:gamma-constants}
\Gamma^\theta_{m} := \sup_{x \in \Omega} \sup_{s,\tau \in [0,1]} \max_{i,j=1, \dots, d} \left| \frac{\partial^m}{\partial \tau^m} \left( D\eta_\theta(\Phi^\theta_{s,\tau}(x), \tau)D\Phi^\theta_{s,\tau}(x) \right)_{i,j} \right|.
\end{equation}
for $m = 1,2$ and the third term is bounded with \cref{theorem:discretized_ode-flow-bound} and the Runge-Kutta integration bound \cite[Thm.\ 3.1]{SolvingODE} and we obtain
\begin{align}
    \begin{split}
    \|D\Phi^\theta_{kh, (k+1)h}(z) - D\Psi^h_{kh, (k+1)h}(z;\eta_\theta)\|_2
    \leq& \tfrac{h^3}{24} \Gamma^\theta_{2} + \tfrac{dh^3}{8}\Gamma^\theta_{1} \|\eta_\theta\|_{C^1} \\
    &+ h^3d  \|\eta_\theta \|_{C^2}  \tfrac{C_{\theta}}{4}  \left(1 + \tfrac{hd}{2} \|\eta_\theta\|_{C^1} \right).
    \end{split}
\end{align}
Furthermore, the second term in \eqref{eq:spectral-diff-decomposition} is bounded by
\begin{align}
    \begin{split}
    &\|D\Psi_{kh}^h(z;\eta_{\theta}) - D \Psi_{kh}^h(z^h;\eta_\theta)\|_2\\
    &\leq h\left\|D\eta_\theta\left( z+\tfrac{h}{2}\eta_\theta(z,kh),kh+\tfrac{h}{2}\right) - D\eta_\theta\left(z^h+\tfrac{h}{2}\eta_\theta(z^h,kh),kh+\tfrac{h}{2}\right) \right\|_2 
    \\
    &\quad\cdot\left\|I + \tfrac{h}{2}D\eta_\theta(z,kh)\right\|_2\\
    &\quad + h \left\| D\eta_\theta\left( z^h+\tfrac{h}{2}\eta_\theta(z^h,kh),kh+\tfrac{h}{2}\right) \cdot \tfrac{h}{2}\left(D\eta_\theta(z^h,kh)-   D\eta_\theta(z,kh) \right)\right\|_2\\
    &\leq hd  \|\eta_\theta \|_{C^2} \left(1 + \tfrac{hd}{2}\|\eta_\theta\|_{C^1}\right)\left\| z+\tfrac{h}{2}\eta_\theta(z,kh) - y_{kh}-\tfrac{h}{2}\eta_\theta(y_{kh},kh) \right\|_2 \\
    &\quad + \tfrac{h^2}{2} d\|\eta_\theta\|_{C^1} \|\eta_\theta \|_{C^2}  \left\|  y_{kh}-z \right\|_2\\
      &\leq hd  \|\eta_\theta \|_{C^2}  \left(\left(1 + \tfrac{hd}{2}\|\eta_\theta\|_{C^1}\right) \left(1+\tfrac{h}{2}L_{\eta_\theta} \right) 
      + \tfrac{h}{2} d\|\eta_\theta\|_{C^1}\right)\left\|  y_{kh}-z \right\|_2\\
      &\leq  h^3  \left(\left(1+\tfrac{h}{2}d \|\eta_\theta\|_{C^1} \right)^2 
      + \tfrac{h}{2} d\|\eta_\theta\|_{C^1}\right)C_{\theta}  (\e^{ \|\eta_\theta \|_{C^2} }-1).
    \end{split}
\end{align}
Finally, we arrive at the following bound for \eqref{eq:spectral-diff-decomposition}:
\begin{align}
    \begin{split}
    \|D{\Phi}^\theta_{kh, (k+1)h}(z)&- D\Psi^h_{kh}(z^h; \eta_\theta)\|_2 \\
    &\leq h^3 \Big[ \tfrac{1}{24} \Gamma^\theta_{2} + \tfrac{d}{8}\Gamma^\theta_{1}\|\eta_\theta\|_{C^1} +  \|\eta_\theta \|_{C^2}  \tfrac{dC_\theta}{4}  \left(1 + \tfrac{hd}{2}\|\eta_\theta\|_{C^1} \right) \\
    &\quad +    \left(\left(1+\tfrac{h}{2}d\|\eta_\theta\|_{C^1} \right)^2 
      + \tfrac{h}{2} d\|\eta_\theta\|_{C^1}\right)C_\theta  (\e^{ \|\eta_\theta \|_{C^2} }-1) \Big]\\
      &= h^3 C(\theta,d).
    \end{split}
\end{align}
We finally conclude
\begin{align}
    \left| \log(|\det(D\Phi_\theta(y))|) -\log(|\det(D\Psi^h_\theta(y))|)  \right|
    \leq \frac{1}{h} \left(\left( 2 \cdot h^3 C(\theta,d) +1\right)^d -1\right).
\end{align}
\end{proof}

\end{theorem}
Finally, we obtain the following bound on the model error:
\begin{theorem}{(Model Error Bound)}
    \label{theorem:model_error-bound}
    Let $\mu = \Phi^{-1}_*\nu \in \mathcal{T}_{d,k}$ and $\mathcal{H}^h_{L,W,K}$ fulfill \cref{eq:assumed-h-bound} and~\eqref{eq:bound-on-K} and $K > 2(k+1)$. 
    Furthermore, let
    \begin{equation}
    L \geq 7 + 2((k-1)-2)+ \lceil \log_2(d+1) \rceil + 2 (  \lceil  \log_2((d+1)(2k-1)) +  \log_2(\log_2(\|\xi\|_{C^k}))  \rceil+1). 
    \end{equation}
    Then, there exist constants $C_\xi, \hat{C}_\xi$ only depending on $k,d$ and the corresponding vector field $\xi$ to $\Phi$, such that we have
    \begin{align}
        &\inf_{\Phi_\theta^h \in \mathcal{H}^h_{L,W,K}} \mathrm{KL}(\mu \| \left(\Psi^h(\xi_\theta)\right)^{-1}_*\nu) \leq   C_\xi \frac{1}{K^{k-1}} + \widehat{C}_\xi h^2.
    \end{align}
\end{theorem}
\begin{proof}
    For $\mu = \Phi^{-1}_*\nu \in \mathcal{T}$ there exists a $\mathrm{ReQU}$-network $\eta_\theta$ from \cref{theorem:approx_neuralflow} such that we have
\begin{align}
    \begin{split}
    \mathrm{KL}(\mu \| \left(\Psi_\theta^h\right)^{-1}_*\nu) =&    
    \mathrm{KL}(\mu \| \left(\Phi_\theta\right)^{-1}_*\nu)  - \mathrm{KL}(\mu \| \Phi^{-1}_*\nu) \\
    &+  \mathrm{KL}(\mu \| \left(\Psi_\theta^h\right)^{-1}_*\nu) -  \mathrm{KL}(\mu \| \left(\Phi_\theta\right)^{-1}_*\nu).
    \end{split}
\end{align}
All divergences are well-defined since $\Phi$, $\Phi_\theta$ and $\Psi^h(\cdot, \eta_\theta)$ are bijective and the identity outside of $\Omega_d$.
By the change of variables formula and \cref{theorem:approx_neuralflow,theorem:discretized_ode_logdet-bound} we have
\begin{align}
    \begin{split}\label{eq:bound KL discretized}
    d_{KL}(\mu \| \left(\Psi_\theta^h\right)^{-1}_*\nu)
    &\leq \int_{\Omega_d} \left| \log \frac{|\det \, D_y (\Phi_{\theta})^{-1}(y)|}{|\det \, D_y \Phi^{-1}(y)|} \right| 
    + \left| \log \frac{|\det\, D_y (\Phi_\theta)^{-1}(y)|}{|\det D_y \Psi^h(y;\eta_\theta)|}\right| \mathrm{d}\mu(y) \\
    &\leq C_\xi \frac{1}{K^{k-1}} + \frac{1}{h} \left(\left( 2 \cdot h^3 C(\theta,d) +1\right)^d -1\right)
    \end{split}
\end{align}
In order to show that $C(\theta,d)$ can be bounded uniformly in $\theta$, we regard the constants $\Gamma_m^\theta$ defined in \cref{eq:gamma-constants}.
By the triangle inequality to $\xi$, \cref{theorem:approx_neuralflow,C ell bound,eq:estimate_xi_norm}, there are constants $a_i = a_i(k,d,\xi)$ such that $\| \eta_\theta\|_{C^i} \leq a_i$ for $i = 0, \ldots, 3$.
We regard the expression $f_{i,j}(\tau) := [D \eta_\theta(\Phi^\theta_{s,\tau}(y, \tau) \cdot D \Phi^\theta_{s,\tau}(y)]_{i,j}$ appearing in $\Gamma^\theta_m$.
Using the $\theta$-independent the Gronwall bound
\begin{equation}
    |\partial_{y_j} (\Phi^\theta_{s,\tau})_k(y)| \leq \| D \Phi^\theta_{s,\tau}(y)\|_2 \leq \e^{d a_1}
\end{equation}
which implies $|f_{i,j}(\tau)| \leq d a_1 \e^{da_1}$.
Similarly, $f_{i,j}'$ contains second spatial derivatives of $\eta_\theta$ and we have
\begin{equation}
    |f_{i,j}'(\tau)| \leq d [a_2 (d a_0 + 1) + a_1^2] \e^{d a_1} =: \Gamma_1
\end{equation}
as well, as
\begin{equation}
    |f_{i,j}''(\tau)| \leq (d a_0 + 1) [(d^2 a_0 + 1) a_3 \e^{d a_1} + d a_1 a_2 (3 \e^{d a_1} + 1)] + d a_1^3 \e^{d a_1} =: \Gamma_2.
\end{equation}
We can, therefore, upper bound $\Gamma_1^\theta$ and $\Gamma_2^\theta$ uniformly in $\theta$ which allows for $2 C(\theta,d) \leq \widetilde{C}_\xi$ and with $\widehat{C}_\xi := d (d!)  \widetilde{C}_\xi^d$

\begin{align}
    \begin{split}
    \mathrm{KL}(\mu \| \left(\Phi_\theta^h\right)^{-1}_*\nu)
    &\leq C_\xi \frac{1}{K^{k-1}} + \frac{1}{h} \left(\left(  h^3 \widetilde{C}_\xi +1\right)^d -1\right)
    \leq C_\xi \frac{1}{K^{k-1}} + \frac{1}{h} \left(d (d!) h^{3} \widetilde{C}_\xi^d \right)\\
    &\leq C_\xi \frac{1}{K^{k-1}} + \widehat{C}_\xi h^2.
    \end{split}
\end{align}
\end{proof}

\subsection{Generalization Error Bounds}
The universal approximation property allowed us to derive a deterministic bound on the model error.
Meanwhile, under our assumptions, we can only hope to find probabilistic bounds for the generalization error
\begin{equation}
    \varepsilon_n^\mathrm{gen} = \varepsilon_n^\mathrm{gen}(\chi_n) = \sup_{\widetilde{\nu} \in \mathcal{H}_n} \left| \mathbb{E}_{X \sim \mu} [- \log(f_{\widetilde{\nu}}(X))] - \widehat{L}_n(\widetilde{\nu}, \chi_n) \right|.
\end{equation}
The main idea is to apply McDiarmid's concentration inequality to the generalization error to obtain a probabilistic bound for the deviation of the generalization error from its expected value. 
By finding a deterministic upper bound for the expected generalization error using Dudley's inequality, we will then be able to obtain a probabilistic bound for the generalization error.

Throughout this section we will again assume that \cref{eq:assumed-h-bound} holds
and write $\ThetaLW := [-1,1]^q$ for the parameter space.
Let $\mu \in \mathcal{T}_{d,k}$.
We will first show the sub-Gaussian property of the process $\{Z_{\theta, n}\}_{\theta \in \Theta}$ with
\begin{equation}
    Z_{\theta, n} = \mathbb{E}_{X \sim \mu}[- \log(f_{{\mu}_{\theta}}(X))] - \widehat{L}_n({\mu}_{\theta}, \chi_n)
\end{equation}
with respect to the parameter metric $d_\Theta$.
\begin{theorem}
\label{theorem:subgaussian-property}
    $\{Z_{\theta,n}\}_{\theta \in \ThetaLW}$ with $\ThetaLW \subseteq [-1, 1]^q$ is a sub-Gaussian process with variogram proxy $\Lambda_Z^2/n \cdot  d_\Theta$, where $\Lambda_Z = (8 d \Lambda^5 \e^\Lambda)^d$.
\end{theorem}
\begin{proof}
    Let $\theta, \theta' \in \Theta$ and $X \sim \mu$.
    Then, by \cref{lem:logdet-formula}, we have
\begin{align}
    |Z_\theta - Z_{\theta^\prime}| 
    &\leq  2\sup_{x \in  [0,1]^d}\left| \log\left(\left|\det(D\Psi_\theta^h(x))\right|\right) - \log\left(\left|\det(D\Psi_{\theta^\prime}^h(x))\right|\right) \right|.
\end{align}
Similarly to the proof of \cref{theorem:discretized_ode_logdet-bound} we obtain with $y_{kh}^\theta := \Psi_{kh}^h(y; \eta_{\theta})$
\begin{align}
\label{eq: bound log det difference discretized}
    &\left| \log \frac{|\det(D\Psi_\theta^h(y))|)}{|\det(D\Psi^h_{\theta^\prime}(y))| }  \right|
    &\leq \sum_{k=0}^{m-1}(2 \| D\psi_{kh}^h(y^\theta_{kh};\eta_\theta) - D\psi_{kh}^h(y^{\theta^\prime}_{kh};\eta_{\theta'}) \|_2+1)^d
-1.
\end{align}
Due to \cref{cor:dnn-lipschitz-z}, we have $\|D \eta_\theta(z)\|_2 \leq \Lambda$
which yields by triangle decomposition with $t \geq 0$ and $x, y \in \Omega$
\begin{align}
    \begin{split}
    \| [D \Psi_{t}^h(\cdot; \eta_\theta)](x) &- [D \Psi_t^h(\cdot; \eta_{\theta'})](x)\|_2
    =
     \|D\psi_{t}^h(x; \eta_\theta) - D \psi_{t}^h(x; \eta_{\theta'}) \|_2 \\
     \leq& \frac{h^2}{2} \Lip(\eta_\theta) \widetilde{\Lambda}_{D \eta}^\Theta \| \theta - \theta'\|_2 \\
     &+ \frac{h^2}{2} \widetilde{\Lambda}_{D\eta}^\Omega \|\eta_\theta (x, t) - \eta_{\theta'}(x, t)\|_2 (1 + \tfrac{h}{2} \Lip(D\eta_{\theta'}))  \\
     \leq& \frac{h^2}{2} \left( \widetilde{\Lambda}_\eta^\Omega \widetilde{\Lambda}_{D \eta}^\Theta + \widetilde{\Lambda}_{D\eta}^\Omega \widetilde{\Lambda}_\eta^\Theta (1 + \tfrac{h}{2} \widetilde{\Lambda}_\eta^\Omega) \right) \| \theta - \theta'\|_2 \\
     \leq& \tfrac{h^2}{2} [\Lambda^5 + \Lambda^4 (1 + \tfrac{h}{2} \Lambda)] \|\theta - \theta'\|_2
    \leq h\Lambda^4 \|\theta - \theta'\|_2.
    \end{split}
\end{align}
as well as
\begin{align}
    \begin{split}
    \| D \psi_{t}^{h} (x; \eta_{\theta}) -& D \psi_{t}^{h}(x; \eta_{\theta})\|_{2} \\
    \leq& h \Lip (\eta_{\theta}) \tfrac{h}{2} \Lip(D \eta_{\theta}) \|x - y\|_{2} \\
    &+ h \Lip(D \eta_{\theta}) \|x + \tfrac{h}{2} \eta_{\theta}(x) - y - \tfrac{h}{2} \eta_{\theta}(y)\|_{2} (1 + \tfrac{h}{2} \Lip(\eta_{\theta}) ) \\
	\leq& h ( \tfrac{h}{2} \widetilde{\Lambda}_{\eta}^{\Omega} \widetilde{\Lambda}_{D\eta}^{\Omega} + \widetilde{\Lambda}_{D \eta}^{\Omega} (1 + \tfrac{h}{2} \widetilde{\Lambda}_{\eta}^{\Omega})^{2}) \|x - y\|_{2} \\
	\leq& h (\tfrac{h}{2} \Lambda^{3} + \Lambda^{2} (1 + \tfrac{h}{2} \Lambda)^{2}) \|x - y\|_2.
    \end{split}
\end{align}
From the latter, we find for any $k = 0, \ldots, m-1$
\begin{align}
    \begin{split}
    \| [D \Psi_{kh, \theta}^{h}]&(\Psi_{0, kh}^{h}(x; \eta_{\theta}))- [D \Psi_{kh, \theta}^{h}](\Psi_{0, kh}^{h}(x; \eta_{\theta'}))\|_{2} \\
=& h \|[D \psi_{kh}^{h}(\cdot; \eta_{\theta})] (\Psi_{0, kh}^{h}(x; \eta_{\theta})) - [D \psi_{kh}^{h}(\cdot; \eta_{\theta})] (\Psi_{0, kh}^{h}(x; \eta_{\theta'}))\|_{2} \\
\leq& h^{2} ( \tfrac{h}{2} \widetilde{\Lambda}_{\eta}^{\Omega} \widetilde{\Lambda}_{D\eta}^{\Omega} + \widetilde{\Lambda}_{D \eta}^{\Omega} (1 + \tfrac{h}{2} \widetilde{\Lambda}_{\eta}^{\Omega})^{2}) \|\Psi_{0, kh}^{h}(x; \eta_{\theta}) - \Psi_{0, kh}^{h}(x; \eta_{\theta'})\|_{2} \\
\leq& h^{2} ( \tfrac{h}{2} \widetilde{\Lambda}_{\eta}^{\Omega} \widetilde{\Lambda}_{D\eta}^{\Omega} + \widetilde{\Lambda}_{D \eta}^{\Omega} (1 + \tfrac{h}{2} \widetilde{\Lambda}_{\eta}^{\Omega})^{2}) \e^{\widetilde{\Lambda}_{\eta}^{\Omega}} \Lambda_{\Psi}^{\Theta} \|\theta- \theta'\|_{2} \\
\leq& h^2 [\tfrac{h}{2} \Lambda^3 + \Lambda^2 (1 + \tfrac{h}{2} \Lambda)^2] \e^{\Lambda} h \Lambda^2 (1 + \tfrac{h}{2}\Lambda) \|\theta - \theta'\|_2 \\
\leq& h \Lambda^2 \e^{\Lambda} \|\theta - \theta'\|_2.
    \end{split}
\end{align}
These two results together allow us to bound
\begin{align}
    \begin{split}
    \| D \Psi_{kh,\theta}^{h}(\Psi_{0, kh}^{h}(x; \eta_{\theta})) -& D \Psi_{kh,\theta'}^{h}(\Psi_{0, kh}^{h}(x; \eta_{\theta'})) \|_{2}\\
\leq& [h \Lambda^{2} \e^{\Lambda} + h \Lambda^{4}] \| \theta - \theta' \|_{2}
\leq 2 h \Lambda^{4} \e^{\Lambda} \| \theta- \theta' \|_{2}
    \end{split}
\end{align}

Finally, this allows us to bound the log-determinant difference
\begin{align}
    \begin{split}
    \big| \log(|\det(D\Psi_\theta^h(x))|) -&\log(|\det(D\Psi^h_{\theta^\prime}(x))|)  \big| \\
    \leq&\sum_{k=0}^{m-1} \left[2 \| D \Psi_{kh,\theta}^{h}(\Psi_{0, kh}^{h(x; \eta_{\theta}))}- D \Psi_{kh, \theta'}^{h}(\Psi_{0, kh}^{h}(y; \eta_{\theta})) \|_{2} + 1\right]^{d} - 1 \\
    \leq& \sum_{k=0}^{m-1} \left[4 h \Lambda^{4} \e^{\Lambda} \| \theta - \theta' \|_{2} + 1\right]^{d} - 1 \\
    \leq& m \sum_{j=1}^{d} \binom{d}{j} [4 h \Lambda^{4} \e^{\Lambda}]^{j} \mathrm{diam}(\ThetaLW)^{j-1} \| \theta - \theta' \|_{2} \\
    \leq& \frac{1}{2} d (d!) 8^{d} \sqrt{2 W^{2}L}^{d} (\Lambda^{4} \e^{\Lambda})^{d} \| \theta- \theta' \|_{2} \\
    \leq& \frac{1}{2} (8 d \Lambda^5 \e^{\Lambda})^d \|\theta - \theta'\|_2
    =: \frac{1}{2} \Lambda_Z \|\theta - \theta'\|_2,
    \end{split}
\end{align}
where $x \in \Omega_d$ and we have used that $\mathrm{diam}(\ThetaLW) \leq 2 W^2 L$ because in each layer we have less than $W$ weights each in the weight matrix and the bias and also that $\sqrt{2 W^2 L} \leq \Lambda$.
This yields
\begin{align}
    \begin{split}
      |Z_\theta - Z_{\theta^\prime}| &\leq   2  \sup_{x \in  \Omega_d}\left| \log\left(\left|\det(D\Psi_\theta^h(x))\right|\right) - \log\left(\left|\det(D\Psi_{\theta^\prime}^h(x))\right|\right) \right|\\
      &\leq \Lambda_Z \|\theta-\theta^\prime\|_2 = \Lambda_Z \cdot d_\Theta(\theta,\theta^\prime).
    \end{split}
\end{align}
Hoeffding's inequality (Lemma 4.5 in \cite{Shalev-Shwartz_Ben-David_2014}) completes the proof.
\end{proof}
Now, we can apply Dudley's inequality to obtain an upper bound for the expected generalization error.
\begin{theorem}{(Bound on Expected Generalization Error)}\label{theorem:bound-generalization-error}
    With $\Lambda_Z$ as in the lemma above, it holds for all $n \in \N$ and $\mathcal{H} = \mathcal{H}^h_{L,W}$ with \cref{eq:assumed-h-bound} fulfilled, that 
    \begin{equation}
     \mathbb{E}[\varepsilon_{\text{gen}}(n)] \leq 48  L W^2 \frac{\Lambda_Z}{\sqrt{n}}.
    \end{equation}
\end{theorem}
\begin{proof}
    As $\ThetaLW \subseteq [-1,1]^q$ with $q \leq L(W+1)W$ is a $\|\cdot\|_2$-compact metric space, $\mathcal{H}$ is a $d_\Theta$-compact space. Obviously, $Z_{\theta,n}$ has continuous paths. With the above lemma this shows that all of the conditions of Dudley's inequality are fulfilled. For the covering number we obtain 
\begin{equation}
N(\varepsilon, d_\Theta, \mathcal{H}) \leq \left(1+ \frac{2\sqrt{q}}{\varepsilon} \right)^q 
\end{equation}
as in Example 13.7 (i) in \cite{MFML}, which can be proven using the comparison of volumes of multidimensional-balls. It suffices to show that the packing number is less or equal than $\left(1+ \frac{2\sqrt{q}}{\varepsilon} \right)^q$, as the covering number obviously cannot be larger than the packing number. By definition of the packing number $D(\varepsilon)$ there exist $\theta_1, \dots, \theta_{D(\varepsilon)}$ centers in $\Theta$ with distance of at least $\varepsilon$ to each other. The $\frac{\varepsilon}{2}$-balls around these centers are thus disjoint and contained in $\overline{B}_{\sqrt{q}+\frac{\varepsilon}{2}}(0)$ since $\Theta \subset \overline{B}_{\sqrt{q}}(0)$. Using the formula for the volume of multidimensional-balls and the fact that the combined volume of all of the $\frac{\varepsilon}{2}$-balls around the centers cannot exceed the volume of the $(\sqrt{q}+\frac{\varepsilon}{2})$-ball, we obtain
\begin{equation}
D(\varepsilon, d_\Theta, \mathcal{H}) \leq \left(1+ \frac{2\sqrt{q}}{\varepsilon} \right)^q. 
\end{equation}
According to Dudley's inequality~\cite[Cor.\ 5.25]{van2014probability} we obtain
\begin{align}
    \begin{split}
    \mathbb{E}[\varepsilon_{\mathrm{gen}}(n)] &= \mathbb{E}[\sup_{\theta \in \Theta} Z_{\theta,n}] 
    \leq 12 \frac{\Lambda_Z}{\sqrt{n}}\int_0^\infty \sqrt{\log(N(\varepsilon, d_\Theta, \mathcal{H}))} \;d\varepsilon \\
    &\leq \frac{12 \Lambda_Z q}{\sqrt{n}} \int_0^{1}\sqrt{\log\left(1+ \frac{2}{u} \right)} \;du
    \leq \frac{12 \sqrt{2} q \Lambda_Z}{\sqrt{n}} \int_0^1 \frac{1}{\sqrt{u}} \, \mathrm{d} u 
    \leq \frac{48 L W^2 \Lambda_Z}{\sqrt{n}} . 
    \end{split}
\end{align}
\end{proof}
Finally, we can apply McDiarmid's concentration inequality~\cite[Thm.\ 3.11]{van2014probability} to obtain the following probabilistic bound:

\begin{theorem}{(Concentration Bound for Generalization Error)}\label{theorem:prob_bound-generalization-error}
For $\varepsilon > 0$ it holds that
\begin{equation}
    \Pr\left(\varepsilon_{\mathrm{gen}}(n) - 48 L W^2 \frac{\Lambda_Z}{\sqrt{n}} > \varepsilon\right) \leq \e^{-\frac{1}{4}\frac{\varepsilon^2n}{D^2}},
\end{equation}
with $D := D(d, \Lambda) = (4 d \Lambda^2)^d$.
\end{theorem}
\begin{proof}
    To apply McDiarmid's inequality we consider the function $f$ defined as
\begin{equation}
f(y_1, \dots, y_n) =  \sup_{\theta \in \ThetaLW} |f_\theta(y_1, \dots, y_n)|,
\end{equation}
where 
\begin{equation}
f_\theta(y_1, \dots, y_n) = \mathbb{E}_{Y \sim \mu } \left[ - \log(f(Y, \theta, h)) \right]  + \frac{1}{n} \sum_{j=1}^n   \log( |\det(D\Psi^h(y_j; \eta_\theta))|)  .
\end{equation}

Next, we need to find an upper bound for 
\begin{equation}
\|D_j^-f(Y_1, \dots, Y_n)\|_\infty := f(Y_1,\dots, Y_n) - \inf_{y_j \in \Omega_d} f(Y_1, \dots, Y_{j-1}, y_j, Y_{j+1}, \dots, Y_n). 
\end{equation}
First, note that
\begin{equation}
f(y_1, \ldots, y_n) = \sup_{\theta \in \ThetaLW} \left| f_\theta(\widetilde{y}_1, y_2, \ldots, y_n) + \tfrac{1}{n} \frac{\log |\det(D \Psi^h(y_1;\eta_\theta))|}{\log |\det(D \Psi^h(\widetilde{y}_1; \eta_\theta))|} \right|.
\end{equation}
Using a similar approach as in (\ref{eq: bound log det difference discretized}) and the Lipschitz property of $D(\Psi^h_{0,kh}(\cdot ; \eta_\theta))$, we obtain for $y_1, \dots, y_n, \widetilde{y}_1 \in \Omega_d$ that
\begin{align}
    \begin{split}
   &\hspace{-1cm} \left| \frac{\log |\det(D \Psi^h(y_1;\eta_\theta))|}{\log |\det(D \Psi^h(\widetilde{y}_1; \eta_\theta))|} \right|\\
& \leq\sum_{k=0}^{m-1} \left[2 \| D \Psi_{kh,\theta}^{h}\left(\Psi_{0,kh}^{h}(y_1;\eta_{\theta}\right)- D \Psi_{kh,\theta}^{h}\left(\Psi_{0,kh}^{h}(\widetilde{y}_{1};\eta_{\theta}\right) \|_{2} + 1\right]^{d} - 1 \\
\leq& \sum_{k=0}^{m-1}\left[2 h (\tfrac{h}{2} \Lambda^{3} + \Lambda^{2}(1 + \tfrac{h}{2} \Lambda)^{2}) \mathrm{diam}(\Omega_{d}) + 1\right]^{d} - 1 \\
\leq& m \sum_{j=1}^{d} \binom{d}{j} (2 h \cdot 2 \Lambda^{2} \cdot \mathrm{diam}(\Omega_{d}))^{j}
\leq (4 d \Lambda^{2})^{d} =: D,
    \end{split}
\end{align}
where we used the fact that for $y \in \Omega_d$ we have $\Psi^h_{0,kh}(y,\eta_\theta) \in \Omega_d$.
Hence, we have that $f(y_1, \dots, y_n) -  f(\widetilde{y}_1, y_2, \dots, y_n) \leq \frac{1}{n}D$,
i.e., the oscillation of $f$ in the first argument can be bounded by $\frac{1}{n}D$. The same upper bound can be derived for all other components and $\|D_j^-f(Y_1, \dots, Y_n)\|_\infty \leq \frac{1}{n}D$ for $j=1,\dots,n$.
Applying McDiarmid's concentration inequality~\cite[Thm.\ 3.11]{van2014probability} concludes the proof.
\end{proof}

\subsection{Learnability of $\mathcal{T}_{d,k}$}
Finally, we are able to show that $\mathcal{T}_{d,k}$ is learnable. For this we combine the bounds for the model and generalization error from the last two sections to obtain the following main theorem:

\begin{theorem}{(Learnability of $\mathcal{T}_{d,k}$)}
\label{theorem:learnability}
    Let $\mu \in \mathcal{T}_{d,k}$, $p \in (0,1)$ and $\varepsilon > 0$.
    Then, a sequence of ERM learners $(\Psi_{\theta_n}^{h_n})_*^{-1} \nu \in \mathcal{H}_{L_n, W_n, K_n}^{h_n}$ with $h_{n} = 1 / \left \lceil 2 \ln (n ^{\frac{1 - p}{8d}}) \right\rceil$, $W_n = \left\lceil (\ln n^{p/2d}) / 2 \sqrt{R_0}\right\rceil$, 
    \begin{equation}
    K_{n} = \left \lfloor \frac{1}{3} \left(\frac{W_{n}}{48 (d+1)}\right) ^{\frac{1}{d+1}} \right\rfloor, \quad \textnormal{and} \quad L_{n} = \left \lfloor \log_{4} \log_{2 \sqrt{R_{0}} W_{n}} \ln n ^{\frac{1 - p}{8d}} \right\rfloor
    \end{equation}
    achieves $\lim_{n \to \infty} \Pr (d_\mathrm{KL}(\mu \|(\Psi_{\theta_n}^{h_n})_*^{-1} \nu) > \varepsilon) = 0. $
\end{theorem}
\begin{proof}
    Our choice of $K_n$ ensures sufficient width of the models and for $\omega_n := 2 \sqrt{R_0} W_n$, we have that
    \begin{equation}
        \log_{\omega_{n}} (n^{\frac{1-p}{4d}}) = \frac{\ln n^{\frac{1-p}{4d}}}{\ln \omega_n}
        \geq \frac{\ln n^{\frac{1-p}{4d}}}{\ln (\ln n^{\frac{p}{2d}})},
    \end{equation}
    so with $n$ large enough, $L_n$ allows for sufficient model capacity to apply \cref{theorem:approx_neuralflow}. 
    Our choice of $h_n$ ensures invertibility of the discretized neural flow by
    \begin{equation}
        \Lambda_n := \omega_n^{L_n 2^{L_n}} < \omega_n^{4^{L_n}} < \ln n^{\frac{1-p}{8d}} \leq \frac{1}{2 h_n}.
    \end{equation}
    Hence, the model error bound \cref{theorem:model_error-bound} applies with
    \begin{equation}
        \varepsilon_n^\mathrm{model} \leq C_\xi \frac{1}{K_n^{k+1}} + \widehat{C}_\xi h_n^2 =: \overline{\varepsilon}_n.
    \end{equation}
    With $C_n^\mathrm{gen} := (8d \Lambda_n^6 \e^{\Lambda_n})^d$ and $D_n = (4d \Lambda_n^2)^{d}$, the concentration bound \cref{theorem:bound-generalization-error} implies
    \begin{align}
        \Pr \left( d_\mathrm{KL}(\mu \| (\Psi_{\theta_n, n}^{h_n})_*^{-1}\nu) > \varepsilon \right) 
        \leq& \Pr \left(\varepsilon^\mathrm{gen}_n - \frac{C_n^\mathrm{gen}}{\sqrt{n}} > \frac{1}{2} (\varepsilon - \overline{\varepsilon}_{n}) - \frac{C_{n}^\mathrm{gen}}{\sqrt{n}}\right) \\
        \leq& \exp \left(- \frac{n}{4D_n^{2}} \left[\frac{1}{2} (\varepsilon - \overline{\varepsilon}_n) - C_{n}^\mathrm{gen} 1/\sqrt{n}\right]^{2}\right).
    \end{align}
    Since $\overline{\varepsilon}_n$ is a null sequence, we will finally have $\overline{\varepsilon}_n < \varepsilon/2$ and, therefore
    \begin{align}
        \left(\frac{1}{2}(\varepsilon - \overline{\varepsilon}_{n}) - \frac{C_{n}^\mathrm{gen}}{\sqrt{n}}\right)^{2} 
         &= \frac{1}{4} (\varepsilon - \overline{\varepsilon}_{n})^{2} - (\varepsilon - \overline{\varepsilon}_{n}) \frac{C_{n}^\mathrm{gen}}{\sqrt{n}} + \frac{(C_{n}^\mathrm{gen})^{2}}{n}\\
         &\geq \frac{1}{4} (\varepsilon - \overline{\varepsilon}_{n})^{2} + \frac{(C_{n}^\mathrm{gen})^{2}}{n}
         \geq \frac{\varepsilon^{2}}{16} + \frac{(C_{n}^\mathrm{gen})^{2}}{n}.
    \end{align}
    With $\Lambda_n < \ln n^\frac{1-p}{8d} < n^\frac{1-p}{4d}$, we have that $n / \Lambda_n^{4d} \ge n^p$.
    On the other hand, $(C_n^\mathrm{gen})^2 /D_n^2$ $= (2 \Lambda_n^4 \e^{\Lambda_n})^{2d}$ and for $n$ large enough, it is $\sqrt{\log_{\omega_n} n^{\frac{1-p}{4d}}} > \log_{\omega_n} \ln n^{\frac{p}{2d}}$ so that 
    \begin{equation}
        \Lambda_n^4 \e^{\Lambda_n} \geq \e^{\Lambda_n} \geq \exp(\omega_n^{2^{L_n}}) \geq n^{\frac{p}{2d}}.
    \end{equation}
    Finally, this leads to 
    \begin{equation}
        \frac{n}{D_n^2} \left(\frac{1}{2}(\epsilon - \overline{\epsilon}_{n}) - \frac{C_{g}}{\sqrt{n}}\right)^{2} \geq \left( \frac{\varepsilon^2}{16 (4d)^{2d}} + 2^{2d}\right) n^p
    \end{equation}
    which concludes the proof.
\end{proof}

\subsection{PAC-Learnability of $\mathcal{T}_{d,k,H}$}
Finally, we want to consider PAC-learn\-ability. While learnability denotes the convergence in probability of $\mathrm{KL}(\mu \| \widehat{\mu}_n)$ to zero for any target measure $\mu$ and $\widehat{\mu}_n$ the corresponding ERM-learner, we seek a PAC bound on the data requirement for learning a target space $\mathcal{T}$.
To show PAC learnability, we fix a maximal $C^k$-norm for the vector fields generating the flow endpoints of the target space. Thus, we define
\begin{align}
    \mathcal{T}_{d,k,H}
    = \left\{ \mu = [\Phi(\xi)]^{-1}_*\nu_0 \in \mathcal{T}_{d,k} : \|\xi\|_{C^k} \leq H \right\}
\end{align}
Note, that for this new target space the constants $C_\xi$ and $\widehat{C}_\xi$ can be upper bounded by constants $C_H$ and $\widehat{C}_H$, which are independent of $\xi$.   

\begin{theorem}\label{theorem:pac-learnability}
    Let $p, \varepsilon, \delta \in (0,1)$.
    With $W_n,L_n,K_n,h_n$ as in Theorem \ref{theorem:learnability} the target space $\mathcal{T}_{d,k,H}$ is PAC-learnable, i.e. it holds that 
    $
    \Pr(\mathrm{KL}(\mu \| (\Psi_{\theta_n,n}^{h_n})^{-1}_*\nu) > \varepsilon) \leq \delta
    $
    for all $n \geq n(\varepsilon, \delta)$ with 
    \begin{equation}
        n(\varepsilon, \delta) = \left\lceil \max \left\{ \widetilde{N}, \exp \left( \left(\tfrac{2c(p)}{\varepsilon}\right)^{\min \left\{2, \tfrac{k-1}{d+1}\right\}^{-1}} \right), \left( \frac{\ln 1/\delta}{\frac{\varepsilon^2}{16 (4d)^{2d}} + 2^{2d}}\right)^{\frac{1}{p}} \right\} \right\rceil
    \end{equation}
    where
    \begin{equation}
        c(p) := \left[ 3^{k-1} \left( \frac{192 d (d+1) \sqrt{R_0}}{p} \right)^{\frac{k-1}{d+1}}C_H + \frac{1}{4} \left( \frac{8d}{1-p}\right)^2 \widehat{C}_H \right].
    \end{equation}
    and $\widetilde{N} \in \mathbb{N}$ is independent of both, $\varepsilon$ and $\delta$.
\end{theorem}
\begin{proof}
    To guarantee sufficient model capacity in \cref{theorem:learnability}, we require $n$ large enough such that
    \begin{align}
        L_n \geq& 7 + 2((k-1)-2)+ \lceil \log_2(d+1) \rceil \\ &+ 2 (  \lceil  \log_2((d+1)(2k-1)) +  \log_2(\log_2(H))  \rceil+1)
    \end{align}
    which can be achieved by some $N_L \in \mathbb{N}$ independently of $\varepsilon$ and $\delta$.
    By our parameter choices, we find for the model error bound
    \begin{align}
        \overline{\varepsilon}_n
        \leq& \frac{3^{k-1} (96 (d+1) \sqrt{R_0})^{\frac{k-1}{d+1}}}{(\ln n^{\frac{p}{2d}})^{\frac{k-1}{d+1}}} C_\xi + \frac{1}{4} \frac{1}{(\ln n^{\frac{1-p}{8d}})^2} \widehat{C}_\xi 
        \leq c (\ln n)^{- \min\left\{2, \tfrac{k-1}{d+1}\right\}} \leq \frac{\varepsilon}{2}
    \end{align}
    for $n \geq \exp \left( \left(\tfrac{2c(p)}{\varepsilon}\right)^{\min \left\{2, \tfrac{k-1}{d+1}\right\}^{-1}} \right)$.
    In order to ensure rate $n^p$, we required 
    \begin{equation}
        \sqrt{\log_{\omega_n} n^{\frac{1-p}{4d}}} > \log_{\omega_n} \ln n^{\frac{p}{2d}}
        \iff \omega_n > \left(\ln n^{\frac{p}{2d}}\right)^{\frac{\ln (\ln n^{\frac{p}{2d}})}{\ln n^{\frac{1-p}{4d}}}}
    \end{equation}
    One sufficient condition for this is
    \begin{equation}
        n > \left( \frac{2d}{p} \right)^{\frac{1-p}{4d - (1-p)}} =: N_\omega \implies \frac{1-p}{4d} \ln n > \ln \left( \frac{p}{2d} \right) + \ln n > \ln \left( \frac{p}{2d} \ln n\right).
    \end{equation}
    Again, we have a lower sample size bound independent of $\varepsilon$ and $\delta$, so we set $\widetilde{N}:= \max\{N_L, N_\omega\}$.
    Finally, the PAC criterion is fulfilled for 
    \begin{align}
        \Pr \left( \mathrm{KL}(\mu \| (\Psi_{\theta_n, n}^{h_n})_*^{-1}\nu) > \varepsilon \right) 
        \leq& \exp \left(\left[ \frac{\varepsilon^2}{16 (4d)^{2d}} + 2^{2d}\right] n^p\right) \leq \delta \\
        \iff n \geq& \left( \frac{\ln 1/\delta}{\frac{\varepsilon^2}{16 (4d)^{2d}} + 2^{2d}}\right)^{\frac{1}{p}}.
    \end{align}
\end{proof}
Our derived sample requirement grows exponentially in $1/\varepsilon$ which is inferior to the algebraic rates derived in \cite{marzouk2024distribution} and a consequence of our treatment of the numerical Runge-Kutta time integration.
Standard error estimates of the integration scheme lead to Lipschitz constants growing with $\e^\Lambda$.

\paragraph{Acknowledgments}
The authors thank G. Steidl, Sebastian Neumayer and Sven Wang for interesting discussions.
This work was supported by a fellowship of the German Academic Exchange Service (DAAD).

\bibliographystyle{siamplain}
\bibliography{references}

\end{document}